\documentclass{article}

\usepackage[utf8]{inputenc}
\usepackage[T1]{fontenc}
\usepackage[margin=1in]{geometry}
\usepackage{amsmath,amssymb,amsthm}
\usepackage{graphicx}
\usepackage[colorlinks=true,allcolors=blue]{hyperref}
\usepackage{booktabs}
\usepackage{mathtools}
\usepackage{tikz}
\usetikzlibrary{arrows.meta,positioning}

\newtheorem{theorem}{Theorem}[section]
\newtheorem{lemma}[theorem]{Lemma}
\newtheorem{proposition}[theorem]{Proposition}
\newtheorem{corollary}[theorem]{Corollary}
\newtheorem{definition}{Definition}[section]
\newtheorem{remark}{Remark}[section]
\newtheorem{assumption}{Assumption}[section]

\DeclareMathOperator{\Tr}{Tr}

\title{Persistent Gaussian Perturbations Prevent Oversmoothing in Recurrent Graph Neural Networks}

\author{Mostafa Haghir Chehreghani\\
	Department of Computer Engineering\\
	Amirkabir University of Technology (Tehran Polytechnic)\\
	Tehran, Iran\\
	\texttt{mostafa.chehreghani@aut.ac.ir}}

\date{}

\begin{document}
	
	\maketitle
	
\begin{abstract}
	
	Oversmoothing is a fundamental limitation of deep graph neural networks
	(GNNs), where repeated message passing causes node representations to
	become increasingly similar, eventually collapsing toward a
	low-dimensional subspace. This phenomenon limits the effective depth of
	message-passing architectures and motivates the search for mechanisms
	that preserve representation diversity.
	In this paper, we study a recurrent graph neural network in which
	independent Gaussian noise is injected after every propagation step and
	analyze the resulting architecture as a stochastic dynamical system.
	Under a standard global contraction assumption on the deterministic
	update, we prove that the hidden representations form a geometrically
	ergodic Markov chain admitting a unique invariant probability measure.
	Our main theoretical result establishes an explicit positive lower bound
	on the expected stationary Dirichlet energy, proportional to both the
	noise variance and the spectral gap of the underlying graph.
	Consequently, the stationary representations cannot collapse onto the
	constant manifold, providing a rigorous guarantee that asymptotic
	oversmoothing is prevented in the sense of non-vanishing Dirichlet
	energy.
	
	Our analysis reveals persistent stochastic perturbations as a
	fundamentally different mechanism for combating oversmoothing, complementing
	existing deterministic approaches based on residual connections,
	normalization, and graph rewiring. Finally, numerical experiments on
	both linear and nonlinear recurrent graph neural networks closely match
	the theoretical predictions, illustrating the emergence of a stationary
	distribution and the predicted dependence of the limiting Dirichlet
	energy on the noise intensity.
	
\end{abstract}
	
	
\section{Introduction}

Graph neural networks (GNNs) learn node representations through repeated
applications of local message-passing operators.
As the number of layers increases, information propagates across
progressively larger neighborhoods, enabling the network to model
long-range interactions.
However, repeated message passing also gives rise to
\emph{oversmoothing}: node representations become increasingly similar,
eventually collapsing onto a low-dimensional subspace or, in many
architectures, onto the space of constant vectors
\cite{li2018deeper,oono2020graph}.
This phenomenon limits the effective depth of GNNs and has become one of
the central theoretical challenges in graph representation learning.
Understanding whether oversmoothing is an unavoidable consequence of
deep propagation, or whether it can be prevented by principled design,
remains an active area of research.

Existing approaches alleviate oversmoothing through architectural
modifications such as residual connections, normalization,
graph rewiring, diffusion-based propagation, or stochastic
regularization.
These methods often improve empirical performance and enable
substantially deeper models, yet their guarantees are typically tied to
specific architectures or finite-depth settings.
Only recently have rigorous results established conditions under which
certain deterministic architectures provably avoid oversmoothing
\cite{scholkemper2025}.
This naturally raises a complementary theoretical question:

\begin{quote}
	Can persistent stochastic perturbations fundamentally change the
	asymptotic dynamics of recurrent graph neural networks so that
	representation collapse becomes impossible?
\end{quote}

In this paper we answer this question affirmatively.
We study a recurrent graph neural network in which independent Gaussian
perturbations are injected after every message-passing step.
Rather than viewing noise merely as a regularization device during
training, we interpret the resulting architecture as a stochastic
dynamical system evolving on the representation space.
This perspective makes it possible to analyze its long-term behavior
using tools from Markov chain theory, random dynamical systems,
optimal transport, and spectral graph theory.

Our main theorem shows that additive Gaussian noise fundamentally changes
the asymptotic behavior of recurrent graph propagation.
Under a standard global contraction assumption on the deterministic
update, the hidden representations converge in distribution toward a
unique invariant probability measure.
Moreover, the expected stationary Dirichlet energy satisfies an explicit
lower bound proportional to
\[
\sigma^2\lambda_2,
\]
where $\sigma^2$ denotes the noise variance and $\lambda_2$ is the
spectral gap of the graph Laplacian.
Consequently, the stationary representations cannot collapse onto the
constant manifold, implying that oversmoothing, interpreted as
asymptotic vanishing of Dirichlet energy, is impossible under the
assumptions of the theorem.

Our analysis is intentionally geometric rather than statistical.
We do not study optimization, convergence of training algorithms,
or predictive performance on downstream tasks.
Instead, we identify a fundamental mechanism by which persistent
stochastic perturbations alter the limiting dynamics of message passing
and prevent degeneration of node representations.
The resulting theory provides a mathematically explicit relationship
between the injected noise, the graph spectrum, and the asymptotic
representation geometry.

\paragraph{Why this theorem matters for practice.}
Although our analysis is theoretical, it has direct implications for the
design of deep and recurrent graph neural networks.
Modern GNN architectures increasingly rely on recurrent propagation,
implicit layers, equilibrium models, or continuous-time dynamics, all of
which involve repeated application of graph operators and are therefore
susceptible to oversmoothing.
Our results demonstrate that persistent stochastic perturbations are not
merely a heuristic regularization technique but fundamentally alter the
limiting behavior of the dynamics by maintaining a strictly positive
level of representation diversity.
The explicit dependence of the stationary Dirichlet energy on both the
noise variance and the graph spectral gap also provides theoretical
guidance for selecting noise levels in practical architectures.
More broadly, the analysis suggests that stochasticity can serve as a
principled mechanism for controlling long-term representation collapse
rather than simply improving optimization or generalization.

The contributions of this paper are as follows.

\begin{enumerate}
	
	\item
	We formulate noisy recurrent graph neural networks as stochastic
	dynamical systems and analyze their asymptotic behavior.
	
	\item
	Under a global contraction assumption on the deterministic dynamics, we
	prove existence and uniqueness of an invariant probability measure,
	together with geometric ergodicity of the resulting Markov process.
	
	\item
	We derive an explicit lower bound on the stationary Dirichlet energy,
	showing that it depends quantitatively on both the graph spectral gap
	and the injected noise variance.
	
	\item
	We prove that asymptotic oversmoothing, characterized by vanishing
	Dirichlet energy, cannot occur under the assumptions of the theorem.
	
	\item
	We validate the theoretical predictions through three numerical
	experiments, including explicit linear dynamics, nonlinear recurrent
	graph neural networks, and empirical verification of the predicted
	quadratic scaling of stationary Dirichlet energy with the noise
	variance.
	
\end{enumerate}

The remainder of the paper is organized as follows.
Section~\ref{sec:relatedwork} reviews the most relevant literature on oversmoothing,
stochastic graph neural networks, dynamical-system formulations of GNNs,
and random dynamical systems.
Section~\ref{sec:mathematical_setting} introduces the mathematical framework, notation, and the
problem formulation adopted throughout the paper.
Section~\ref{sec:noisy_dynamics} presents the noisy recurrent graph neural network together with
its stochastic dynamical-system interpretation.
Section~\ref{sec:Asymptotic_Analysis} develops the theoretical analysis, proving the existence and
uniqueness of the stationary distribution, geometric ergodicity, and the
positive lower bound on the stationary Dirichlet energy.
Section~\ref{sec:Numerical_Experiments} provides numerical experiments that validate the theoretical
predictions and illustrate the behavior of noisy recurrent GNNs.
Finally, Sections~\ref{sec:Discussion} and~\ref{sec:Conclusion} discuss the implications and limitations of the
results and conclude the paper.

\section{Related Work}
\label{sec:relatedwork}

Graph neural networks (GNNs) have become one of the dominant paradigms
for learning from graph-structured data, with successful applications
in node classification, graph classification, recommendation,
molecular modeling, and scientific computing
\cite{kipf2017,hamilton2017,velivckovic2018,xu2019gin,HaghirChehreghani2022Nature}.
Their rapid development has also revealed several fundamental
theoretical limitations, among which oversmoothing and oversquashing
have emerged as two of the most significant obstacles to building deep
and expressive graph neural networks.

\paragraph{Oversmoothing in graph neural networks.}

Oversmoothing is one of the principal limitations preventing graph neural
networks (GNNs) from scaling to large propagation depths.
Li et al.~\cite{li2018deeper} first interpreted graph convolution as repeated
Laplacian smoothing, explaining why node representations gradually become
indistinguishable.
Oono and Suzuki~\cite{oono2020graph} later proved that deep GNNs
exponentially lose expressive power under broad assumptions.
Subsequent theoretical and empirical studies further characterized
oversmoothing and its relationship with graph topology
\cite{nt2020revisiting,elinas2023revisiting}.

Many architectural modifications have been proposed to alleviate this
phenomenon.
Representative examples include PairNorm
\cite{zhao2020pairnorm},
DropEdge
\cite{rong2020dropedge},
Jumping Knowledge Networks
\cite{xu2018jknet},
APPNP
\cite{klicpera2019predict},
GCNII
\cite{chen2020gcnii},
GraphNorm
\cite{cai2021graphnorm},
DeeperGCN
\cite{li2020deepergcn},
and residual normalization techniques.
Most recently,
Scholkemper et al.~\cite{scholkemper2025}
proved that residual connections together with normalization can
provably prevent oversmoothing under suitable assumptions.
Hoseinnia \emph{et al.}~\cite{hoseinnia2025}
proposed Adaptive Early Embedding together with a biased DropEdge
mechanism to mitigate oversmoothing during node classification.
Zohrabi \emph{et al.}~\cite{zohrabi2024}
extended graph neighborhoods using centrality- and similarity-based
criteria to improve information propagation while reducing excessive
feature homogenization.
Unlike these architectural approaches,
our work leaves the propagation operator unchanged and instead studies the
effect of persistent stochastic perturbations.
We show that additive Gaussian noise alone guarantees a positive asymptotic
Dirichlet energy independently of network depth.

\paragraph{Oversquashing and graph geometry.}

Oversquashing is a complementary phenomenon in which exponentially many
long-range messages are compressed into fixed-dimensional node
representations, leading to severe information bottlenecks
\cite{alon2021bottleneck,topping2022understanding,digiovanni2023oversquashing}.
Recent work has proposed architectural modifications, graph rewiring, and
similarity-aware neighborhood augmentation to alleviate these bottlenecks.
For example, Mohamadi and Haghir Chehreghani
\cite{mohamadi2025oversquashing} mitigate over-squashing in graph few-shot
learning by combining local and global similarity information to enrich
message propagation.
Although over-squashing and over-smoothing are distinct phenomena, both
highlight the importance of preserving informative representations in deep
message-passing networks.

\paragraph{Stochastic graph neural networks.}

Randomness has been incorporated into graph neural networks in many forms,
including stochastic diffusion,
feature perturbation,
edge sampling,
and data augmentation.
GRAND
\cite{chamberlain2021grand}
models message passing through stochastic differential equations,
whereas Random Features
\cite{sato2021random},
DropEdge
\cite{rong2020dropedge},
and stochastic regularization methods
\cite{godwin2022simple}
improve generalization by injecting randomness during training.
Neural SDE formulations of graph learning
\cite{zhuang2020sdgnn}
further establish continuous-time stochastic message passing.
Several recent works have also enriched graph representations by
incorporating additional semantic information.
For example, Content Augmented Graph Neural Networks
(CAGNN) \cite{nasrabadi2025}
integrate content-based information into graph message passing to
improve representation quality.
These methods primarily pursue improved prediction accuracy.
In contrast,
our objective is to characterize the long-time stochastic dynamics of
recurrent propagation and to derive explicit depth-independent lower
bounds on representation diversity.

\paragraph{Continuous-time, equilibrium, and dynamical-system models.}

Continuous-depth formulations provide another perspective on deep graph
learning.
Graph Neural ODEs
\cite{poli2019graphode}
and Neural ODE based architectures
\cite{chen2018neuralode}
replace discrete layers by continuous dynamical systems.
Deep Equilibrium Models
\cite{bai2019deep}
and Implicit Graph Neural Networks
\cite{gu2020implicit}
replace finite-depth propagation by fixed-point equations.
Graph-Coupled Oscillator Networks (GraphCON)
\cite{rusch2023graphcon}
introduce second-order oscillatory dynamics that empirically stabilize
deep propagation and alleviate oversmoothing.
Our framework differs fundamentally from these approaches.
Rather than modifying the underlying dynamics,
we analyze stochastic contractive systems and prove that additive Gaussian
noise generates a unique invariant probability measure with non-vanishing
stationary Dirichlet energy.

\paragraph{Graph Transformers and modern deep graph architectures.}

Recent graph transformer architectures extend message passing through global
attention mechanisms.
Graphormer
\cite{ying2021graphormer}
and GraphGPS
\cite{rampasek2022graphgps}
combine transformer layers with structural positional encodings and local
message passing, substantially improving long-range representation learning.
Although these architectures mitigate several limitations of classical
message-passing GNNs, they do not provide theoretical guarantees preventing
asymptotic feature collapse.
Our results are complementary and may also be applicable to recurrent
propagation modules embedded within transformer-based graph architectures.

\paragraph{Spectral graph theory.}

Our analysis is fundamentally spectral.
The Dirichlet energy is expressed through the graph Laplacian,
connecting our results with classical spectral graph theory
\cite{chung1997}
and graph signal processing
\cite{shuman2013}.
Modern spectral analyses of graph neural networks
\cite{balcilar2021analyzing}
have emphasized the frequency response of message-passing operators.
Our lower bounds provide a probabilistic interpretation of spectral
energy preservation under stochastic propagation.

\paragraph{Random dynamical systems and invariant measures.}

The theoretical analysis builds upon the theory of random dynamical systems,
iterated random functions,
Markov chains,
and optimal transport.
Classical existence and uniqueness results for invariant measures
\cite{diaconis1999iterated,arnold1998,meyn2009}
combine naturally with Wasserstein contraction techniques
\cite{villani2003,villani2009,hairer2011}
to characterize the long-term behaviour of noisy graph propagation.
These probabilistic tools enable explicit lower bounds on stationary
Dirichlet energy.

\paragraph{Positioning of the present work.}

To the best of our knowledge,
this paper is the first to establish that persistent additive Gaussian
perturbations alone provably prevent asymptotic oversmoothing in recurrent
graph neural networks.
Under a global contraction assumption,
the noisy dynamics admit a unique invariant probability measure with
strictly positive stationary variance.
Combining this stochastic characterization with spectral estimates of the
graph Laplacian yields explicit positive lower bounds on the limiting
Dirichlet energy that hold independently of propagation depth.

\section{Mathematical Setting}
\label{sec:mathematical_setting}

The purpose of this section is to introduce the graph-theoretic notation,
the representation space, and the notion of oversmoothing used throughout
the paper.  No assumptions on the dynamics are imposed here; these will be
introduced in Section~\ref{sec:noisy_dynamics}.
Throughout the paper, all vector spaces are equipped with the Euclidean norm
$\|\cdot\|$, and all matrix spaces with the Frobenius norm
$\|\cdot\|_{F}$.
Expectation is taken with respect to the underlying probability space
$(\Omega,\mathcal{F},\mathbb{P})$ on which all random variables are
defined.

\subsection{Graphs and node representations}

Let $G=(V,E)$
be a finite, connected, undirected graph with
$|V|=N$ and $|E|=m$. The adjacency matrix is denoted by
$A\in\mathbb{R}^{N\times N}$,
and the degree matrix by
\[
D=\operatorname{diag}(d_1,\ldots,d_N).
\]
The (unnormalized) graph Laplacian is
$
L=D-A$. 
Since the graph is connected,
\[
0=\lambda_1(L)
<
\lambda_2(L)
\le
\cdots
\le
\lambda_N(L),
\]
where $\lambda_2(L)$ denotes the algebraic connectivity (spectral gap).

Each node carries a feature vector in $\mathbb{R}^{d}$.
Collecting the node features row-wise yields the representation matrix

\[
H=
\begin{bmatrix}
	h_1^{\top}\\
	\vdots\\
	h_N^{\top}
\end{bmatrix}
\in
\mathbb{R}^{N\times d}.
\]
The initial representation is denoted by
$
H^{(0)}=X$,
where $X$ is the input feature matrix.

\subsection{The zero-mean representation}

Oversmoothing concerns the collapse of node representations toward the
subspace of constant vectors.
It is therefore convenient to decompose every representation into its
constant and non-constant components.
Let
\[
P
=
I_N
-
\frac1N
\mathbf1\mathbf1^{\top}
\]
be the orthogonal projection onto the subspace orthogonal to
$\operatorname{span}\{\mathbf1\}$.
For every representation matrix $H$, define
$
\bar H
=
PH$.
Equivalently,

\[
\bar H
=
H
-
\frac1N
\mathbf1\mathbf1^{\top}H.
\]

Thus,

\[
H
=
\frac1N
\mathbf1\mathbf1^{\top}H
+
\bar H,
\]
where the first term is constant across all nodes and
$\bar H$ contains the node-dependent component.
The quantity
$
\|\bar H\|_{F}^{2}
$
measures the total squared deviation of the node representations from their mean.

\subsection{Projected Gaussian noise}

The orthogonal projection introduced above also determines the statistical
properties of the projected noise that appears throughout the subsequent
analysis.

\begin{lemma}[Projected Gaussian noise]
	\label{lem:projected_noise}
	
	Let
	\[
	P = I_N - \frac1N \mathbf1\mathbf1^\top,
	\]
	and let
	\[
	\Xi \in \mathbb R^{N\times d}
	\]
	be a random matrix whose entries are independent standard Gaussian random variables,
	\[
	\Xi_{ij} \stackrel{\mathrm{i.i.d.}}{\sim} \mathcal N(0,1).
	\]
	Define the projected noise $\bar\Xi = P\Xi$.
	Then the following properties hold.
	\begin{enumerate}
		\item The projected noise is centered, i.e., $\mathbb E[\bar\Xi] = 0$.
		\item Each feature column has covariance:
		\[
		\operatorname{Cov}(\bar\Xi_{:,k}) = P, \qquad k=1,\dots,d.
		\]
		\item The expected squared Frobenius norm satisfies:
		\[
		\boxed{\mathbb E\|\bar\Xi\|_F^2 = (N-1)d.}
		\]
	\end{enumerate}
\end{lemma}

\begin{proof}
	Since   $\bar\Xi = P\Xi$ and $\mathbb E[\Xi]=0$, linearity gives $\mathbb E[\bar\Xi]=0$.	
	For each $k=1,\dots,d$, the $k$-th column $\Xi_{:,k}$ is distributed as $\mathcal N(0,I_N)$. Therefore,
	\[
	\bar\Xi_{:,k} = P\,\Xi_{:,k} \sim \mathcal N(0,PIP) = \mathcal N(0,P),
	\]
	where we used the symmetry and idempotence of $P$, i.e., $P^2 = P = P^\top$.
	Hence $\operatorname{Cov}(\bar\Xi_{:,k}) = P$.
	
	Finally,
	\[
	\mathbb E\|\bar\Xi\|_F^2 = \sum_{k=1}^{d} \mathbb E\|\bar\Xi_{:,k}\|^2
	= \sum_{k=1}^{d} \operatorname{Tr}(P)
	= d\operatorname{Tr}(P).
	\]
	The projection matrix $P = I_N - \frac1N\mathbf1\mathbf1^\top$ has eigenvalues
	$1$ (multiplicity $N-1$) and $0$ (multiplicity $1$), so $\operatorname{Tr}(P) = N-1$.
	Therefore:
	\[\mathbb E\|\bar\Xi\|_F^2 = (N-1)d. \]
\end{proof}

\subsection{Dirichlet energy}

To quantify the degree of smoothness of a representation we use its
Dirichlet energy.

\begin{definition}[Dirichlet energy]	
	For a representation matrix
	$H\in\mathbb{R}^{N\times d}$,
	
	\[
	\mathcal E(H)
	=
	\frac1m
	\operatorname{\Tr}
	(H^{\top}LH).
	\]	
\end{definition}

Since
$
L\mathbf1=0
$, 
the energy depends only on the projected representation.
Indeed,
\[
\mathcal E(H)
=
\frac1m
\operatorname{\Tr}
(\bar H^{\top}L\bar H).
\]
Equivalently,
\[
\mathcal E(H)
=
\frac1m
\sum_{(i,j)\in E}
\|h_i-h_j\|^2.
\]
Thus, the Dirichlet energy measures the average discrepancy between
neighboring node representations.
The following elementary characterization will be used repeatedly.

\begin{proposition}
	\label{prop:energy_zero}
	
	Assume that $G$ is connected.
	Then,
	$
	\mathcal E(H)=0
	$	
	if and only if all node representations are identical.
	
\end{proposition}

\begin{proof}
	
	Since $L$ is positive semidefinite,
	
	\[
	\mathcal E(H)\ge0.
	\]
	
	Moreover,
	
	\[
	\ker(L)
	=
	\operatorname{span}\{\mathbf1\}.
	\]
	
	Hence,
	
	\[
	\mathcal E(H)=0
	\]
	
	if and only if
	$
	\bar H=0$,	
	which is equivalent to all rows of $H$ being identical.	
\end{proof}

\subsection{Oversmoothing}

We now formalize the notion of oversmoothing adopted in this paper.
Several notions of oversmoothing have been proposed in the literature (e.g., based on feature variance, pairwise distances, or representation rank). In this paper we adopt the following energy-based definition.
This definition is particularly convenient because the Dirichlet energy is
directly linked to the graph Laplacian and admits spectral lower bounds.

\begin{definition}[Oversmoothing]
	\label{def:oversmoothing}
	
	Let
	
	\[
	(H^{(t)})_{t\ge0}
	\]
	
	be a stochastic sequence of node representations.	
	The sequence is said to exhibit
	\emph{oversmoothing}
	if
	
	\[
	\lim_{t\rightarrow\infty}
	\mathbb E
	\left[
	\mathcal E(H^{(t)})
	\right]
	=
	0.
	\]
	
\end{definition}

The following definition is central to the paper.

\begin{definition}[Escape from oversmoothing]
	\label{def:escape}
	The sequence
	$(H^{(t)})$
	is said to
	\emph{escape oversmoothing}
	if there exists a constant
	
	\[
	\delta>0
	\]
	
	such that
	
	\[
	\liminf_{t\rightarrow\infty}
	\mathbb E
	\left[
	\mathcal E(H^{(t)})
	\right]
	\ge
	\delta.
	\]
	
\end{definition}

This definition is purely geometric.
It concerns the asymptotic behavior of the hidden representations and does
not refer to optimization, statistical generalization, or downstream task
performance.

\subsection{Spectral lower bound}

The following inequality connects the Dirichlet energy with the projected
representation and will play a central role in the analysis.

\begin{lemma}
	\label{lem:spectral_gap}
	
	For every
	$H\in\mathbb R^{N\times d}$,
	
	\[
	\mathcal E(H)
	\ge
	\frac{\lambda_2(L)}{m}
	\,
	\|\bar H\|_F^2.
	\]
	
\end{lemma}

\begin{proof}
	
	Since	
	\[
	\bar H
	\perp
	\ker(L),
	\]	
	the Courant--Fischer characterization of the second eigenvalue gives	
	\[
	\operatorname{\Tr}
	(
	\bar H^{\top}
	L
	\bar H
	)
	\ge
	\lambda_2(L)
	\,
	\|\bar H\|_F^2.
	\]
	Dividing by $m$ yields the result.
	
\end{proof}

\section{Noisy Recurrent Graph Neural Networks}
\label{sec:noisy_dynamics}

In this section we formulate the noisy recurrent graph neural network
studied throughout the paper.
Rather than viewing the architecture as a deep feed-forward network, we
interpret it as a discrete-time stochastic dynamical system evolving on
the representation space
$\mathbb{R}^{N\times d}$.
This perspective naturally leads to a Markov process whose long-term
behavior can be analyzed using tools from stochastic stability and
ergodic theory.

\subsection{Deterministic recurrent dynamics}

Let
\[
F:\mathbb{R}^{N\times d}
\longrightarrow
\mathbb{R}^{N\times d}
\]
denote the deterministic update map of a recurrent graph neural network.
The precise form of $F$ is not prescribed.
It may consist of any combination of graph convolutions,
attention mechanisms,
normalization layers,
nonlinear activations,
or residual connections,
provided that the assumptions introduced below are satisfied.

Starting from
\(
H^{(0)}=X
\),
the deterministic recurrent dynamics are
\[
H^{(t+1)}
=
F(H^{(t)}),
\qquad
t\ge0.
\]
Without stochastic perturbations,
repeated application of $F$
may progressively reduce the variation of the node representations,
eventually leading to oversmoothing.

\subsection{Noisy recurrent dynamics}

To counteract this collapse,
we inject an independent Gaussian perturbation after every recurrent
update.

Specifically,
\[
H^{(t+1)}
=
F(H^{(t)})
+
\sigma\,\Xi^{(t)},
\]
where
\(
\sigma>0
\)
is the noise intensity and
\(
\Xi^{(t)}
\in
\mathbb{R}^{N\times d}
\)
has independent standard Gaussian entries:
\[
\Xi^{(t)}_{ij}
\sim
\mathcal N(0,1).
\]

The sequence
\[
(\Xi^{(t)})_{t\ge0}
\]
is assumed to be independent and identically distributed and independent
of the initial representation.
The additive perturbation is generated from isotropic Gaussian noise acting on all nodes and feature dimensions. After projection, the noise is confined to the zero-mean subspace.
Its role is not to model uncertainty in the data,
but to continuously inject variation into the recurrent dynamics.

The qualitative difference between the deterministic and noisy dynamics is illustrated in Figure~\ref{fig:overview}. In the absence of noise, repeated contraction drives the representations toward the consensus manifold, whereas additive noise continuously injects variation, leading to a stationary distribution away from consensus.

\begin{figure}
	\centering
	
	\begin{tikzpicture}[
>=Stealth,
		node distance=1.3cm,
		every node/.style={font=\small},
		state/.style={circle,draw,fill=black,inner sep=1.5pt},
		title/.style={font=\bfseries}
		]
		
		
		\node[title] at (0,3.6) {Deterministic dynamics};
		
		\node[state] (a1) at (0,2.8) {};
		\node[state] (a2) at (0,1.9) {};
		\node[state] (a3) at (0,1.0) {};
		\node[state] (a4) at (0,0.1) {};
		
		\draw[->] (a1)--(a2);
		\draw[->] (a2)--(a3);
		\draw[->] (a3)--(a4);
		
		\draw[dashed,thick] (-1.2,-0.7)--(1.2,-0.7);
		
		\node at (0,-1.2) {Consensus manifold};
		\node at (0,-1.8) {$\mathcal E(H)\rightarrow0$};
		
		
		\node[title] at (7,3.6) {Noisy dynamics};
		
		\node[state] (b1) at (7,2.8) {};
		\node[state] (b2) at (6.3,1.9) {};
		\node[state] (b3) at (7.7,1.9) {};
		\node[state] (b4) at (6.4,1.0) {};
		\node[state] (b5) at (7.6,1.0) {};
		\node[state] (b6) at (7,0.1) {};
		
		\draw[->] (b1)--(b2);
		\draw[->] (b1)--(b3);
		\draw[->] (b2)--(b4);
		\draw[->] (b2)--(b5);
		\draw[->] (b3)--(b4);
		\draw[->] (b3)--(b5);
		\draw[->] (b4)--(b6);
		\draw[->] (b5)--(b6);
		
		\draw[dashed,thick] (5.8,-0.7) ellipse (1.6 and 0.7);
		
		\node at (7,-1.5) {Invariant distribution};
		\node at (7,-2.1) {$\mathcal E(H)>0$};
		
	\end{tikzpicture}
	
	\caption{Conceptual illustration of the long-term behavior of recurrent graph neural networks. Left: without additive noise, repeated contraction drives the hidden representations toward the consensus manifold, resulting in vanishing Dirichlet energy (oversmoothing). Right: additive Gaussian noise continually injects variability, producing a stationary distribution on the zero-mean subspace with strictly positive expected Dirichlet energy, thereby preventing asymptotic feature collapse.}
	\label{fig:overview}
	
\end{figure}
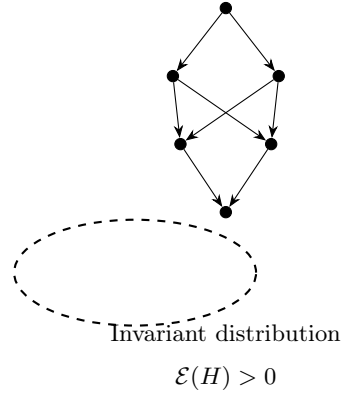

\subsection{Projection onto the zero-mean subspace}

Since the graph Laplacian annihilates constant vectors,
only the projected representations
\[
\bar H^{(t)}
=
PH^{(t)}
\]
are relevant for the analysis of oversmoothing.
Applying the projection operator gives:
\[
\bar H^{(t+1)}
=
PF(H^{(t)})
+
\sigma
P\Xi^{(t)}.
\]

Let
\[
\mathcal X=\operatorname{Im}(P)
\]
denote the zero-mean subspace. Since
\[
\mathcal X\subset\mathbb R^{N\times d},
\]
the restriction of
\[
F:\mathbb R^{N\times d}\to\mathbb R^{N\times d}
\]
to \(\mathcal X\) is well defined.

Assumption~\ref{ass:projection} implies that the projected deterministic
update depends only on the projected representation. Hence we define
\[
\Delta:\mathcal X\rightarrow\mathcal X,
\qquad
\Delta(\bar H)=PF(\bar H).
\]

The projected dynamics therefore take the form
\[
\boxed{
	\bar H^{(t+1)}
	=
	\Delta(\bar H^{(t)})
	+
	\sigma\bar\Xi^{(t)}.
}
\]

This representation shows that the evolution of the projected
representations is completely determined within the state space
\(\mathcal X\).
It will be used throughout the subsequent analysis.

\subsection{Standing assumptions}

The theoretical results in the following sections are established under
the assumptions below.

\begin{assumption}[Graph connectivity]
	\label{ass:connected}
	
	The graph
	$G$
	is finite,
	connected,
	and undirected.	
\end{assumption}

Connectivity guarantees that
\(\lambda_2(L)>0\),
which is essential for obtaining positive lower bounds on the Dirichlet
energy.


\begin{assumption}[Projection invariance]
	\label{ass:projection}
	
	The deterministic update satisfies	
	\[
	PF(H)
	=
	PF(PH),
	\qquad
	H\in\mathbb R^{N\times d}.
	\]
	Equivalently, two representations that differ only by a constant vector
	shared by all nodes produce the same projected update.
	
\end{assumption}

Assumption~\ref{ass:projection} requires that the projected component of
the deterministic update depends only on the projected representation.
Equivalently, adding the same constant feature vector to every node does
not affect the evolution of the zero-mean component.
This property is naturally satisfied by a broad class of graph neural
network architectures whose message-passing operations depend only on
relative interactions between neighboring node features.
Examples include linear graph convolution operators, graph diffusion
layers, and recurrent message-passing schemes without global bias terms.

More generally, the assumption is structural rather than restrictive: it
ensures that the consensus subspace
$\operatorname{span}\{\mathbf 1\}$ is invariant under the deterministic
dynamics and therefore allows the projected dynamics to evolve as a
closed system on the state space
$\mathcal X=\operatorname{Im}(P)$.
This invariance is essential for the Markov formulation developed in the
subsequent analysis.


\begin{assumption}[Global contraction]
	\label{ass:contraction}
	
	The mapping	
	\[
	\Delta:\mathcal X\rightarrow\mathcal X
	\]	
	is globally Lipschitz with constant
	\(
	0\le\alpha<1
	\),
	that is,	
	\[
	\|
	\Delta(H_1)-\Delta(H_2)
	\|_F
	\le
	\alpha
	\|
	H_1-H_2
	\|_F
	\]
	for every
	$H_1,H_2\in\mathcal X$.
	
	In addition, the origin is a fixed point of the deterministic dynamics,
	
	\[
	\boxed{\Delta(0)=0.}
	\]
	
\end{assumption}

The contraction property guarantees that the deterministic dynamics
reduce distances on the projected state space.
The additional condition
$\Delta(0)=0$
ensures that the consensus representation is an equilibrium of the
deterministic dynamics, allowing the contraction estimate to be applied
directly to the norm of the projected representations.


\begin{assumption}[Additive noise]
	\label{ass:noise}	
	The random matrices
	$(\Xi^{(t)})_{t\ge0}$
	are independent and identically distributed,
	independent of the initial condition,
	centered,	
	\[
	E[\Xi^{(t)}]=0,
	\]	
	and satisfy
	\[
	E\|\Xi^{(t)}\|_F^2<\infty.
	\]	
\end{assumption}
We additionally assume that the additive noise is Gaussian.


The four assumptions play distinct roles. Graph connectivity provides the spectral properties of the Laplacian. Projection invariance allows the projected dynamics to evolve autonomously. Global contraction ensures ergodicity of the projected Markov chain. Finally, the additive noise supplies persistent stochastic excitation.

\subsection{Markov formulation}

Under Assumption~\ref{ass:projection}, the projected dynamics evolve
autonomously on the zero-mean subspace
\(
\mathcal X
=
\operatorname{Im}(P)
\).
The projected recursion therefore takes the form
\[
\boxed{
	\bar H^{(t+1)}
	=
	\Delta(\bar H^{(t)})
	+
	\sigma\bar\Xi^{(t)}.
}
\]

Since the driving noise is independent across time, this recursion
naturally defines a stochastic dynamical system on the state space
\(\mathcal X\).
The following proposition formalizes this observation.

\begin{proposition}[Projected dynamics define a Markov chain]
	\label{prop:markov}
	
	Suppose Assumption~\ref{ass:projection} holds.
	Assume further that
	\[
	F:\mathbb{R}^{N\times d}\rightarrow\mathbb{R}^{N\times d}
	\]
	is Borel measurable.
	Then the projected process
	\[
	(\bar H^{(t)})_{t\ge0}
	\]
	defined by
	\[
	\bar H^{(t+1)}
	=
	\Delta(\bar H^{(t)})
	+
	\sigma\bar\Xi^{(t)}
	\]
	is a time-homogeneous Markov chain on the state space
	\(
	\mathcal X=\operatorname{Im}(P)
	\).
\end{proposition}

\begin{proof}
	
	Since $P$ is linear and $F$ is Borel measurable,
	the mapping
	\[
\Delta=P\circ (F|_{\mathcal X}):\mathcal X\to\mathcal X,
	\]
	is Borel measurable.
	For every $t\ge0$,	
	\[
	\bar H^{(t+1)}
	=
	\Delta(\bar H^{(t)})
	+
	\sigma\bar\Xi^{(t)},
	\]
	where the random matrices
	$(\bar\Xi^{(t)})_{t\ge0}$
	are independent and identically distributed and are independent of	
	\[
	\sigma\!\left(
	\bar H^{(0)},
	\ldots,
	\bar H^{(t)}
	\right).
	\]
	
	Let
	$A\subset\mathcal X$
	be a Borel set.
	Conditioning on the past gives	
	\[
	\begin{aligned}
		&
		\mathbb P
		\!\left(
		\bar H^{(t+1)}
		\in A
		\,\middle|\,
		\bar H^{(0)},
		\ldots,
		\bar H^{(t)}
		\right)
		\\
		&=
		\mathbb P
		\!\left(
		\Delta(\bar H^{(t)})
		+
		\sigma\bar\Xi^{(t)}
		\in A
		\,\middle|\,
		\bar H^{(0)},
		\ldots,
		\bar H^{(t)}
		\right).
	\end{aligned}
	\]
	Since
	$\bar\Xi^{(t)}$
	is independent of the entire history,
	\[
	\mathbb P
	\!\left(
	\bar H^{(t+1)}
	\in A
	\,\middle|\,
	\bar H^{(0)},
	\ldots,
	\bar H^{(t)}
	\right)
	=
	K(\bar H^{(t)},A),
	\]
	where	
	\[
	K(x,A)
	=
	\mathbb P
	\!\left(
	\Delta(x)
	+
	\sigma\bar\Xi
	\in A
	\right).
	\]
	Thus the conditional distribution of
	$\bar H^{(t+1)}$
	depends only on the current state
	$\bar H^{(t)}$.
	Hence
	$(\bar H^{(t)})_{t\ge0}$
	is a Markov chain.
	Moreover, the kernel $K$ does not depend on $t$ because both the mapping
	$\Delta$
	and the law of the additive noise are time-independent.
	Therefore the chain is time homogeneous.	
\end{proof}

Proposition~\ref{prop:markov} shows that the projected recursion defines
a measurable time-homogeneous Markov chain on the state space
$\mathcal X$.
Its transition kernel is given by
\[
K(x,A)
=
\mathbb P
\left(
\Delta(x)+\sigma\bar\Xi\in A
\right),
\]
for every $x\in\mathcal X$ and every Borel set
$A\subset\mathcal X$.

Proposition \ref{prop:markov} provides the stochastic framework used throughout the remainder of the paper.
Our objective is therefore to characterize the invariant probability
measure of the projected Markov chain and to study the long-term
behavior of the projected hidden representations.

\subsection{Roadmap of the analysis}

The remainder of the analysis proceeds in three steps. We first establish ergodicity of the projected Markov chain. We then derive a stationary second-moment identity and show that additive noise induces persistent stationary variance. Finally, combining this variance estimate with the spectral lower bound yields escape from oversmoothing.

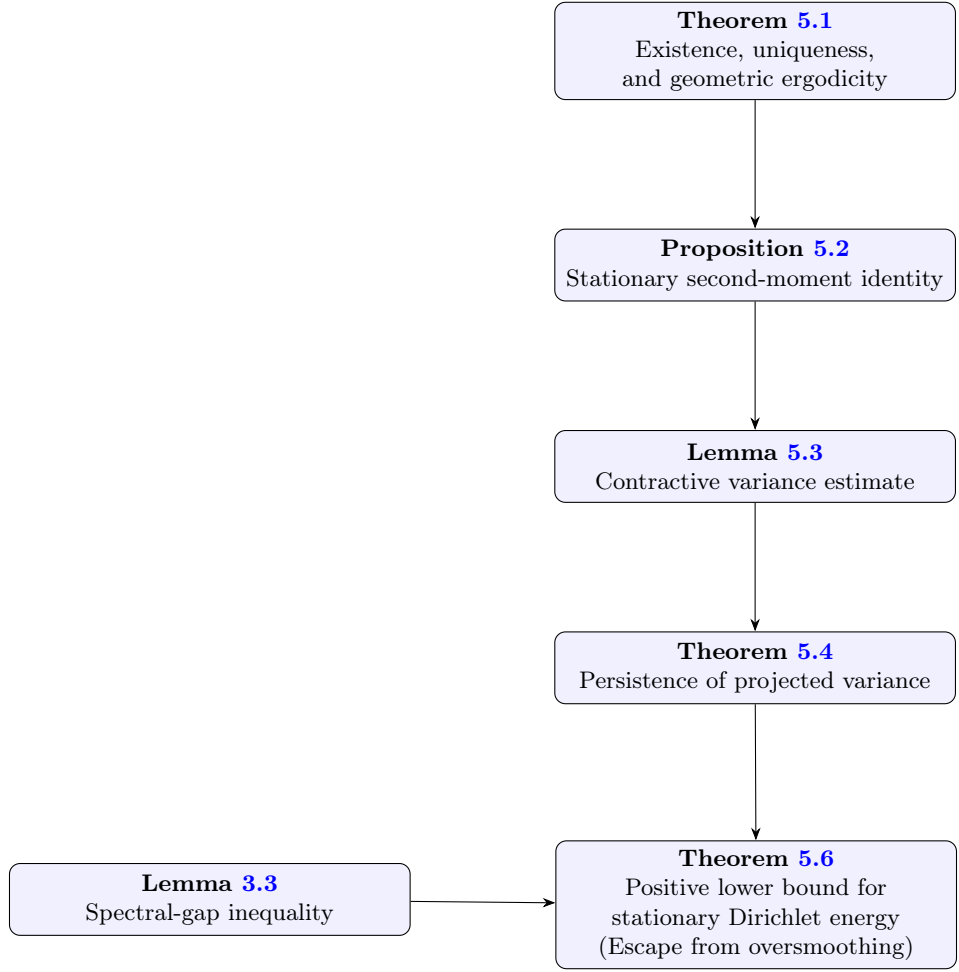
\begin{figure}
	\centering
	
	\begin{tikzpicture}[
		>=Stealth,
		node distance=1.7cm,
		every node/.style={font=\small},
		box/.style={
			draw,
			rounded corners,
			align=center,
			minimum width=5.3cm,
			minimum height=0.95cm,
			fill=blue!6
		}
		]
		
		\node[box] (ergodicity)
		{
			\textbf{Theorem~\ref{thm:ergodicity}}\\
			Existence, uniqueness,\\
			and geometric ergodicity
		};
		
		\node[box,below=of ergodicity] (varianceid)
		{
			\textbf{Proposition~\ref{prop:variance_identity}}\\
			Stationary second-moment identity
		};
		
		\node[box,below=of varianceid] (variance)
		{
			\textbf{Lemma~\ref{lem:variance_contraction}}\\
			Contractive variance estimate
		};
		
		\node[box,below=of variance] (persist)
		{
			\textbf{Theorem~\ref{thm:variance_persistence}}\\
			Persistence of projected variance
		};
		
		\node[box,below left=2.1cm and 1.9cm of persist] (spectral)
		{
			\textbf{Lemma~\ref{lem:spectral_gap}}\\
			Spectral-gap inequality
		};
		
		\node[
		box,
		below right=1.8cm and 1.7cm of persist,
		xshift=-7cm
		] (main)
		{
			\textbf{Theorem~\ref{thm:main}}\\
			Positive lower bound for\\
			stationary Dirichlet energy\\
			(Escape from oversmoothing)
		};
		
		\draw[->] (ergodicity) -- (varianceid);
		\draw[->] (varianceid) -- (variance);
		\draw[->] (variance) -- (persist);
		\draw[->] (persist) -- (main);
		\draw[->] (spectral) -- (main);
		
	\end{tikzpicture}
	
	\caption{
		Logical dependency of the theoretical development.
		Theorem~\ref{thm:ergodicity} establishes the existence, uniqueness, and
		geometric ergodicity of the stationary distribution.
		Proposition~\ref{prop:variance_identity} derives the stationary
		second-moment identity, which is used together with
		Lemma~\ref{lem:variance_contraction} to prove
		Theorem~\ref{thm:variance_persistence}.
		Finally, combining persistence of the projected variance with the
		spectral-gap estimate of
		Lemma~\ref{lem:spectral_gap}
		yields the main result,
		Theorem~\ref{thm:main}, which establishes a strictly positive lower bound
		for the stationary Dirichlet energy and therefore the prevention of
		asymptotic oversmoothing.
	}
	\label{fig:proof_structure}
	
\end{figure}

Figure~\ref{fig:proof_structure} summarizes the logical dependencies
between the principal results proved in the remainder of the paper.
The probabilistic analysis establishes the existence of a unique
stationary distribution together with a strictly positive stationary
variance, while the graph-theoretic estimate converts this persistent
variance into a positive lower bound for the Dirichlet energy, leading
to the main escape-from-oversmoothing theorem.

\section{Asymptotic Analysis}
\label{sec:Asymptotic_Analysis}

\subsection{Existence and uniqueness of the invariant measure}

We begin by establishing the existence of a unique stationary
distribution for the noisy recurrent dynamics.
This result is not specific to graph neural networks; rather, it is a
classical consequence of the theory of contractive iterated random
functions.
The overall proof strategy is illustrated in
Figure~\ref{fig:proof_structure}.
The results of this section are developed in the order indicated there,
culminating in the proof of the main theorem.

Recall that the projected Markov chain introduced in
Section~\ref{sec:noisy_dynamics}
evolves according to
\[
\bar H^{(t+1)}
=
\Delta(\bar H^{(t)})
+
\sigma\bar\Xi^{(t)},
\]
on the state space
$\mathcal X=\operatorname{Im}(P)$.
Since $\mathcal X=\operatorname{Im}(P)$ is a finite-dimensional linear subspace of $\mathbb R^{N\times d}$, it is complete under the Frobenius norm.

The deterministic dynamics contract representations toward consensus, whereas the additive noise continuously injects energy into the projected subspace. The competition between these two effects suggests the existence of a stationary regime. The following theorem formalizes this intuition.

\begin{theorem}[Ergodicity of the projected Markov chain]
	\label{thm:ergodicity}
	
Suppose that Assumptions
\ref{ass:connected},
\ref{ass:projection},
\ref{ass:contraction},
and
\ref{ass:noise}
hold.
	
Then the projected Markov chain
$(\bar H^{(t)})_{t\ge0}$
with transition kernel $K$:
	\begin{enumerate}
		
		\item admits a unique invariant probability measure
		$\mu$;
		
		\item converges to $\mu$ from every initial condition;
		
		\item converges geometrically in the Wasserstein metric.
		
	\end{enumerate}	
\end{theorem}

\begin{proof}	
	The projected dynamics define a contractive iterated random function on the
	complete metric space
	\[
	\mathcal X=\operatorname{Im}(P),
	\]
	endowed with the Frobenius norm.
	By Assumption~\ref{ass:contraction}, the mapping
	\[
	\Delta:\mathcal X\rightarrow\mathcal X
	\]
	is globally Lipschitz with contraction constant
	\(\alpha<1\).
	Moreover, by Assumption~\ref{ass:noise}, the projected noise sequence
	\[
	(\bar\Xi^{(t)})_{t\ge0}
	\]
	is independent and identically distributed with finite second moments.
	
Therefore, the hypotheses of the classical ergodicity theorem for
contractive iterated random functions are satisfied
(see, e.g.,
\cite[Theorem~1]{diaconis1999iterated};
see also
\cite{hairer2011}).
	Consequently, the projected Markov chain admits a unique invariant
	probability measure, converges to it from every initial condition, and
	the convergence is geometric in the Wasserstein distance.	
\end{proof}

\begin{remark}
	\label{rem:ergodicity}
	
	The Gaussian assumption is not essential for
	Theorem~\ref{thm:ergodicity}.
	The result remains valid for any independent additive noise sequence
	having finite second moments.	
	Gaussianity becomes important only in the subsequent analysis,
	where the covariance structure of the noise is used to derive an explicit
	stationary variance identity and the lower bound on the Dirichlet energy.
	
\end{remark}


\subsection{Stationary second-moment identity}
\label{sec:second_moment}

Having established the existence and uniqueness of the invariant
probability measure, we next characterize the stationary second moment
of the projected hidden representations.

Throughout this subsection, let
$\mu$
denote the unique invariant probability measure of
Theorem~\ref{thm:ergodicity}.
If
\(
\bar H\sim\mu
\),
then stationarity implies that
\[
\bar H'
=
\Delta(\bar H)
+
\sigma\bar\Xi
\]
has the same distribution as
$\bar H$,
where
$\bar\Xi=P\Xi$
is independent of
$\bar H$.

The following proposition gives an exact balance relation between the
stationary second moment, the deterministic contraction, and the
variance injected by the additive noise.

\begin{proposition}[Exact stationary second-moment identity]
	\label{prop:variance_identity}
	
Suppose that Assumptions
\ref{ass:connected},
\ref{ass:projection},
\ref{ass:contraction},
and
\ref{ass:noise}
hold.	
	Then	
	\[
	\boxed{
		\mathbb E_\mu
		\!\left[
		\|\bar H\|_F^2
		\right]
		=
		\mathbb E_\mu
		\!\left[
		\|\Delta(\bar H)\|_F^2
		\right]
		+
		\sigma^2 (N-1)d.
	}
	\]	
\end{proposition}

\begin{proof}	
	Since
	$\mu$
	is an invariant probability measure,
	\[
	\bar H
	\stackrel{d}{=}
	\Delta(\bar H)
	+
	\sigma\bar\Xi,
	\]
	where
	$\bar\Xi$
	is independent of
	$\bar H$.	
	Therefore,	
	\[
	\begin{aligned}
		\mathbb E_\mu
		\!\left[
		\|\bar H\|_F^2
		\right]
		&=
		\mathbb E_\mu
		\!\left[
		\|
		\Delta(\bar H)
		+
		\sigma\bar\Xi
		\|_F^2
		\right]
		\\
		&=
		\mathbb E_\mu
		\!\left[
		\|\Delta(\bar H)\|_F^2
		\right]
		+
		2\sigma
		\,
		\mathbb E
		\!\left[
		\left\langle
		\Delta(\bar H),
		\bar\Xi
		\right\rangle_F
		\right]
		\\
		&\qquad
		+
		\sigma^2
		\,
		\mathbb E
		\!\left[
		\|\bar\Xi\|_F^2
		\right].
	\end{aligned}
	\]
	
	It remains to evaluate the last two terms.	
	Since
	$\bar\Xi$
	is independent of
	$\bar H$
	and satisfies	
$
	\mathbb E[\bar\Xi]=0
$
	by Lemma~\ref{lem:projected_noise},	
	\[
	\begin{aligned}
		\mathbb E
		\!\left[
		\left\langle
		\Delta(\bar H),
		\bar\Xi
		\right\rangle_F
		\right]
		&=
		\mathbb E
		\!\left[
		\left\langle
		\Delta(\bar H),
		\mathbb E[\bar\Xi]
		\right\rangle_F
		\right]
		\\
		&=
		0.
	\end{aligned}
	\]
	Furthermore,
	Lemma~\ref{lem:projected_noise}
	establishes that	
	\[
	\mathbb E
	\!\left[
	\|\bar\Xi\|_F^2
	\right]
	=
	(N-1)d.
	\]
	Substituting these two identities into the previous expansion yields	
	\[
	\mathbb E_\mu
	\!\left[
	\|\bar H\|_F^2
	\right]
	=
	\mathbb E_\mu
	\!\left[
	\|\Delta(\bar H)\|_F^2
	\right]
	+
	\sigma^2(N-1)d,
	\]
	which completes the proof.
	
\end{proof}

\subsection{Persistence of stationary variance}
\label{sec:variance_persistence}

The stationary variance identity established in
Proposition~\ref{prop:variance_identity}
shows that additive noise continuously injects variability into the
projected hidden representations.
The next result shows that the deterministic contraction cannot
completely eliminate this injected variance.

We first quantify the effect of the contraction mapping on the stationary
second moment.

\begin{lemma}[Contractive variance estimate]
	\label{lem:variance_contraction}
	
Suppose that Assumptions
\ref{ass:connected},
\ref{ass:projection},
\ref{ass:contraction},
and
\ref{ass:noise}
hold.
	
	Then
	
	\[
	\boxed{
		\mathbb E_\mu
		\!\left[
		\|\Delta(\bar H)\|_F^2
		\right]
		\le
		\alpha^2
		\,
		\mathbb E_\mu
		\!\left[
		\|\bar H\|_F^2
		\right].
	}
	\]
	
\end{lemma}

\begin{proof}
	
	By Assumption~\ref{ass:contraction},
	\[
	\|\Delta(X)-\Delta(Y)\|_F
	\le
	\alpha
	\|X-Y\|_F
	\qquad
	(\alpha<1).
	\]	
	Taking $Y=0$ gives,	
\[
\|
\Delta(H)-\Delta(0)
\|_F
\le
\alpha
\|H-0\|_F.
\]
Since
$\Delta(0)=0$,

\[
\|\Delta(H)\|_F
\le
\alpha
\|H\|_F.
\]
Squaring both sides yields
\[
\|\Delta(H)\|_F^2
\le
\alpha^2
\|H\|_F^2.
\]
Finally, taking expectations with respect to the invariant probability
	measure $\mu$ gives	
	\[
	\mathbb E_\mu
	\!\left[
	\|\Delta(\bar H)\|_F^2
	\right]
	\le
	\alpha^2
	\,
	\mathbb E_\mu
	\!\left[
	\|\bar H\|_F^2
	\right].
	\]
	
\end{proof}

The previous lemma combines naturally with the stationary variance
identity to show that the projected representations necessarily retain a
strictly positive stationary variance.

\begin{theorem}[Positive lower bound on the stationary variance]
	\label{thm:variance_persistence}
	
Suppose that Assumptions
\ref{ass:connected},
\ref{ass:projection},
\ref{ass:contraction},
and
\ref{ass:noise}
hold. Then the unique invariant probability measure satisfies
	
	\[
	\boxed{
		\mathbb E_\mu
		\!\left[
		\|\bar H\|_F^2
		\right]
		\ge
		\frac{\sigma^2 (N-1)d}
		{1-\alpha^2}
		>0.
	}
	\]
	
	Consequently, the projected hidden representations possess a strictly
	positive stationary variance.
	
\end{theorem}

\begin{proof}
	
	By Proposition~\ref{prop:variance_identity},
	\[
	\mathbb E_\mu
	\!\left[
	\|\bar H\|_F^2
	\right]
	=
	\mathbb E_\mu
	\!\left[
	\|\Delta(\bar H)\|_F^2
	\right]
	+
	\sigma^2(N-1)d.
	\]
	Applying Lemma~\ref{lem:variance_contraction} yields
	\[
	\mathbb E_\mu
	\!\left[
	\|\bar H\|_F^2
	\right]
	\le
	\alpha^2
	\mathbb E_\mu
	\!\left[
	\|\bar H\|_F^2
	\right]
	+
	\sigma^2(N-1)d.
	\]
	Rearranging gives	
	\[
	(1-\alpha^2)
	\,
	\mathbb E_\mu
	\!\left[
	\|\bar H\|_F^2
	\right]
	\ge
	\sigma^2(N-1)d.
	\]
	Since
	$\alpha<1$,
	
	\[
	1-\alpha^2>0,
	\]
	and therefore	
	\[
	\mathbb E_\mu
	\!\left[
	\|\bar H\|_F^2
	\right]
	\ge
	\frac{\sigma^2(N-1)d}
	{1-\alpha^2}.
	\]	
	Because
	$\sigma>0$,
	the right-hand side is strictly positive, proving the claim.
	
\end{proof}

\begin{remark}
	
	Theorem~\ref{thm:variance_persistence}
	is purely probabilistic and does not depend on the graph topology.
	It shows that persistent additive noise prevents the projected
	representations from collapsing to zero variance, regardless of the
	particular form of the contractive mapping~$\Delta$.
	
	This result is the fundamental stochastic ingredient of the subsequent
	analysis.
	To relate stationary variance to oversmoothing, one must additionally
	connect the projected variance to the graph Dirichlet energy through the
	spectral properties of the graph Laplacian.
	
\end{remark}

\subsection{Uniform boundedness of transient second moments}

The convergence of expected Dirichlet energies in
Theorem~\ref{thm:main}
requires a uniform bound on the second moments of the transient
representations.

\begin{lemma}[Uniform boundedness of second moments]
	\label{lem:uniform_second_moment}
	
Suppose that Assumptions
\ref{ass:connected},
\ref{ass:projection},
\ref{ass:contraction},
and
\ref{ass:noise}
hold.	
	Then the projected process satisfies
	
	\[
	\sup_{t\ge0}
	\mathbb E
	\!\left[
	\|\bar H^{(t)}\|_F^2
	\right]
	<
	\infty.
	\]
	More precisely,	
	\[
	\mathbb E
	\!\left[
	\|\bar H^{(t)}\|_F^2
	\right]
	\le
	\alpha^{2t}
	\|\bar H^{(0)}\|_F^2
	+
	\frac{\sigma^2 (N-1)d}
	{1-\alpha^2},
	\qquad
	t\ge0.
	\]
	
\end{lemma}

\begin{proof}
	
	The projected dynamics satisfy	
	\[
	\bar H^{(t+1)}
	=
	\Delta(\bar H^{(t)})
	+
	\sigma\bar\Xi^{(t)}.
	\]
	Expanding the squared Frobenius norm gives	
	\[
	\begin{aligned}
		\|\bar H^{(t+1)}\|_F^2
		&=
		\|\Delta(\bar H^{(t)})\|_F^2
		+
		2\sigma
		\langle
		\Delta(\bar H^{(t)}),
		\bar\Xi^{(t)}
		\rangle_F
		\\
		&\qquad
		+
		\sigma^2
		\|\bar\Xi^{(t)}\|_F^2.
	\end{aligned}
	\]
	Taking expectations and using the independence and zero mean of
	\(\bar\Xi^{(t)}\) yields	
	\[
	\mathbb E
	\!\left[
	\|\bar H^{(t+1)}\|_F^2
	\right]
	=
	\mathbb E
	\!\left[
	\|\Delta(\bar H^{(t)})\|_F^2
	\right]
	+
	\sigma^2 (N-1)d,
	\]	
	where Lemma~\ref{lem:projected_noise} has been used.
	
	Since
	\(\Delta(0)=0\)
	and
	\(\Delta\)
	is globally contractive,	
	\[
	\|\Delta(H)\|_F
	\le
	\alpha
	\|H\|_F.
	\]	
	Therefore,	
	\[
	\mathbb E
	\!\left[
	\|\bar H^{(t+1)}\|_F^2
	\right]
	\le
	\alpha^2
	\mathbb E
	\!\left[
	\|\bar H^{(t)}\|_F^2
	\right]
	+
	\sigma^2 (N-1)d.
	\]
	
	Define	
	\[
	M_t
	=
	\mathbb E
	\!\left[
	\|\bar H^{(t)}\|_F^2
	\right].
	\]
	Then	
	\[
	M_{t+1}
	\le
	\alpha^2 M_t
	+
	\sigma^2 (N-1)d.
	\]
	Iterating this inequality gives
	\[
	M_t
	\le
	\alpha^{2t}M_0
	+
	\sigma^2 (N-1)d
	\sum_{k=0}^{t-1}
	\alpha^{2k}.
	\]
	Since
	\(\alpha<1\),
	
	\[
	\sum_{k=0}^{t-1}
	\alpha^{2k}
	\le
	\frac{1}{1-\alpha^2},
	\]
	and hence	
	\[
	M_t
	\le
	\alpha^{2t}
	\|\bar H^{(0)}\|_F^2
	+
	\frac{\sigma^2 (N-1)d}
	{1-\alpha^2}.
	\]
	Taking the supremum over
	\(t\)
	completes the proof.
	
\end{proof}

\subsection{Escape from oversmoothing}
\label{sec:escape}

We are now in a position to establish the principal theoretical result
of the paper.
It states that the noisy recurrent graph neural network cannot converge
to a completely oversmoothed representation.
Instead, the stochastic dynamics admit a stationary regime with strictly
positive expected Dirichlet energy.

\begin{theorem}[Escape from oversmoothing under contractive noisy recurrent dynamics]
	\label{thm:main}
	
Suppose that Assumptions
\ref{ass:connected},
\ref{ass:projection},
\ref{ass:contraction},
and
\ref{ass:noise}
hold.	
	Let	
	\[
	(H^{(t)})_{t\ge0}
	\]
	denote the noisy recurrent graph neural network defined in
	Section~\ref{sec:noisy_dynamics}.
	Then	
\[
\boxed{
	\lim_{t\rightarrow\infty}
	\mathbb E
	\!\left[
	\mathcal E(H^{(t)})
	\right]
	=
	\mathbb E_\mu
	\!\left[
	\mathcal E(\bar H)
	\right]
	\ge
	\frac{\lambda_2(L)}{m}
	\,
	\frac{\sigma^2(N-1)d}
	{1-\alpha^2}
	>0.
}
\]
	Consequently, the noisy recurrent graph neural network escapes
	oversmoothing in the sense of
	Definition~\ref{def:escape}.	
\end{theorem}

At this stage all ingredients shown in
Figure~\ref{fig:proof_structure}
have been established.
The proof of the main theorem is now obtained by combining the
probabilistic variance estimate with the spectral lower bound.

\begin{proof}
	By Theorem~\ref{thm:ergodicity}, the projected chain
	$(\bar H^{(t)})_{t\ge0}$ converges to the unique invariant measure $\mu$
	in the 2‑Wasserstein metric.  The Dirichlet energy satisfies the
	quadratic estimate
	\[
	\mathcal E(H) = \frac1m\operatorname{Tr}(H^{\top}LH)
	\le \frac{\lambda_N(L)}{m}\|H\|_F^2,
	\]
	and it depends only on the projected representation,
	$\mathcal E(H^{(t)}) = \mathcal E(\bar H^{(t)})$.  Because $\mathcal E$ is
	continuous and has at most quadratic growth, convergence in the
	$W_2$ metric implies convergence of expectations (see, e.g.,
	\cite[Chapter~6]{villani2009}).
	Consequently,
	\[
	\lim_{t\to\infty} \mathbb E\bigl[\mathcal E(H^{(t)})\bigr]
	= \mathbb E_\mu\bigl[\mathcal E(\bar H)\bigr].
	\]
	
	By the spectral lower bound (Lemma~\ref{lem:spectral_gap}),
	\[\mathcal E(\bar H) \ge \frac{\lambda_2(L)}{m}\|\bar H\|_F^2.\]
	Taking expectations with respect to $\mu$ and applying the stationary
	variance lower bound from Theorem~\ref{thm:variance_persistence} gives
	\[
	\mathbb E_\mu\bigl[\mathcal E(\bar H)\bigr]
	\ge \frac{\lambda_2(L)}{m}
	\frac{\sigma^2(N-1)d}{1-\alpha^2} \;>\; 0,
	\]
	which completes the proof.
\end{proof}

\begin{corollary}
	
	Under the assumptions of
	Theorem~\ref{thm:main},
	the limiting hidden representations remain separated from the consensus
	manifold in expectation.
	Consequently, stochastic perturbations provide a persistent mechanism
	preventing asymptotic feature collapse.
	
\end{corollary}

The proof separates naturally into three independent components.
	First, ergodicity guarantees the existence of a unique stationary
	distribution and convergence toward it.
	Second, additive noise together with contractive dynamics produces a
	stationary distribution with persistent variability.
	Finally, the spectral lower bound of
	Lemma~\ref{lem:spectral_gap}, together with the positive stationary
	variance established in
	Theorem~\ref{thm:variance_persistence},
	shows that the invariant measure cannot be supported on the consensus
	manifold and therefore has strictly positive stationary Dirichlet
	energy.	
	This decomposition clearly distinguishes the probabilistic,
	dynamical, and graph-theoretic ingredients of the analysis.

\section{Numerical Experiments}
\label{sec:Numerical_Experiments}

The theoretical analysis predicts that persistent additive Gaussian noise
prevents asymptotic oversmoothing by maintaining a strictly positive
stationary Dirichlet energy.
In this section we illustrate these predictions through three numerical
experiments.
The experiments are intentionally simple and are designed to validate the
theoretical results rather than optimize predictive performance.
All simulations were implemented in Python using NumPy, NetworkX, and
PyTorch.


\subsection{Experiment 1: Linear stochastic propagation}

Our first experiment directly validates the stochastic dynamical system
analyzed in Sections~\ref{sec:noisy_dynamics} and
\ref{sec:Asymptotic_Analysis}. The objective is to verify the principal
theoretical prediction that persistent additive Gaussian perturbations
prevent asymptotic oversmoothing by maintaining a strictly positive
stationary Dirichlet energy.

The simulated dynamics are
\[
H^{(t+1)}
=
\alpha \hat A H^{(t)}
+
\sigma\,\Xi^{(t)},
\]
where $\hat A$ is the normalized adjacency matrix,
$\alpha=0.95$ is the contraction coefficient,
and the entries of $\Xi^{(t)}$ are independent standard Gaussian random
variables.

A connected Erd\H{o}s--R\'enyi graph with
$N=100$ vertices and edge probability $0.08$
was generated.
Each node was assigned a $16$-dimensional feature vector initialized by
independent standard Gaussian variables.
The process was simulated for $250$ iterations, which was sufficient for
all trajectories to reach stationarity.
The experiments were performed for eleven noise levels,
\[
\sigma
\in
\{
0,0.02,0.04,\ldots,0.20
\}.
\]

Although propagation is performed using the normalized adjacency matrix
$\hat A$, graph smoothness is evaluated using the combinatorial graph
Laplacian
\[
L=D-A,
\]
which is the quantity appearing throughout the theoretical analysis.
For every noise level we record:
\begin{itemize}
	\item the Dirichlet energy
	\[
	E(H^{(t)})
	=
	\frac{1}{m}\operatorname{Tr}\bigl((H^{(t)})^{\top} L H^{(t)}\bigr),
	\]
	
	\item the projected feature energy
	\[
	\|P H^{(t)}\|_F^2,
	\]
	where	
	\[
	P
	=
	I-\frac1N\mathbf1\mathbf1^\top
	\]
	projects onto the subspace orthogonal to the consensus vector,
	
	\item the stationary Dirichlet energy,
	
	\item the stationary projected feature energy.
\end{itemize}

Figure~\ref{fig:energy_iteration}
shows the evolution of the Dirichlet energy throughout the propagation
process.
The deterministic trajectory ($\sigma=0$) serves as the baseline and
exhibits the strongest decay toward zero, corresponding to complete
oversmoothing.
In contrast, every positive noise level converges to a strictly positive
stationary energy.
Moreover, the limiting energy increases monotonically with the noise
intensity.
These observations agree with
Theorem~\ref{thm:main}, which predicts that additive Gaussian
perturbations maintain a positive asymptotic Dirichlet energy.

Figure~\ref{fig:variance_iteration}
shows the evolution of the projected feature energy
$\|P H^{(t)}\|_F^2$.
The deterministic system contracts toward the consensus subspace,
whereas additive noise continuously injects diversity into the node
representations.
Consequently, the projected variance converges to a positive stationary
value whose magnitude increases with $\sigma$, precisely as predicted by
Theorem~\ref{thm:variance_persistence}.

The stationary quantities are summarized in
Figures~\ref{fig:stationary_energy_sigma}
and~\ref{fig:variance_sigma}.
Both exhibit a clear monotone dependence on the noise level.
Increasing the noise intensity raises both the stationary Dirichlet
energy and the stationary projected feature energy, demonstrating that
persistent stochastic perturbations counteract the progressive loss of
feature diversity caused by repeated graph propagation.
These empirical observations provide direct numerical validation of the
stationary lower bounds established in the theoretical analysis.


\begin{figure}
	\centering
	\includegraphics[width=0.78\linewidth]{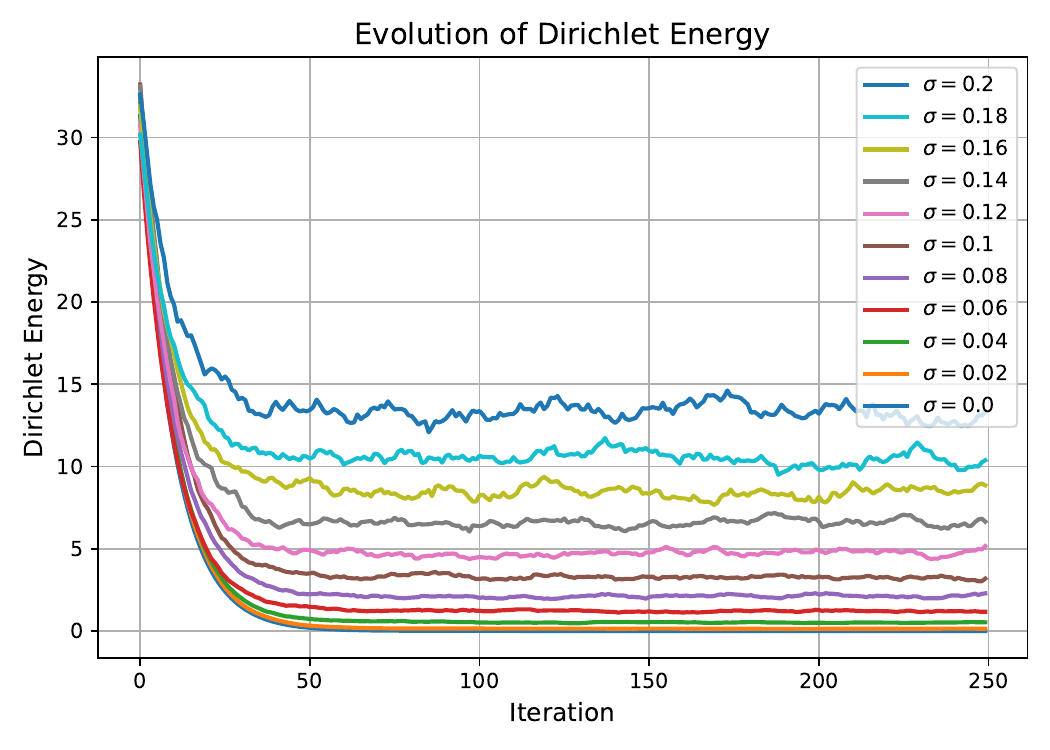}
	\caption{
		Evolution of the Dirichlet energy for the linear stochastic propagation
		model.
		The deterministic trajectory ($\sigma=0$) exhibits complete energy decay,
		whereas every positive noise level converges to a strictly positive
		stationary value.
		The limiting energy increases monotonically with the noise intensity,
		consistent with Theorem~\ref{thm:main}.}
	\label{fig:energy_iteration}
\end{figure}

\begin{figure}
	\centering
	\includegraphics[width=0.78\linewidth]{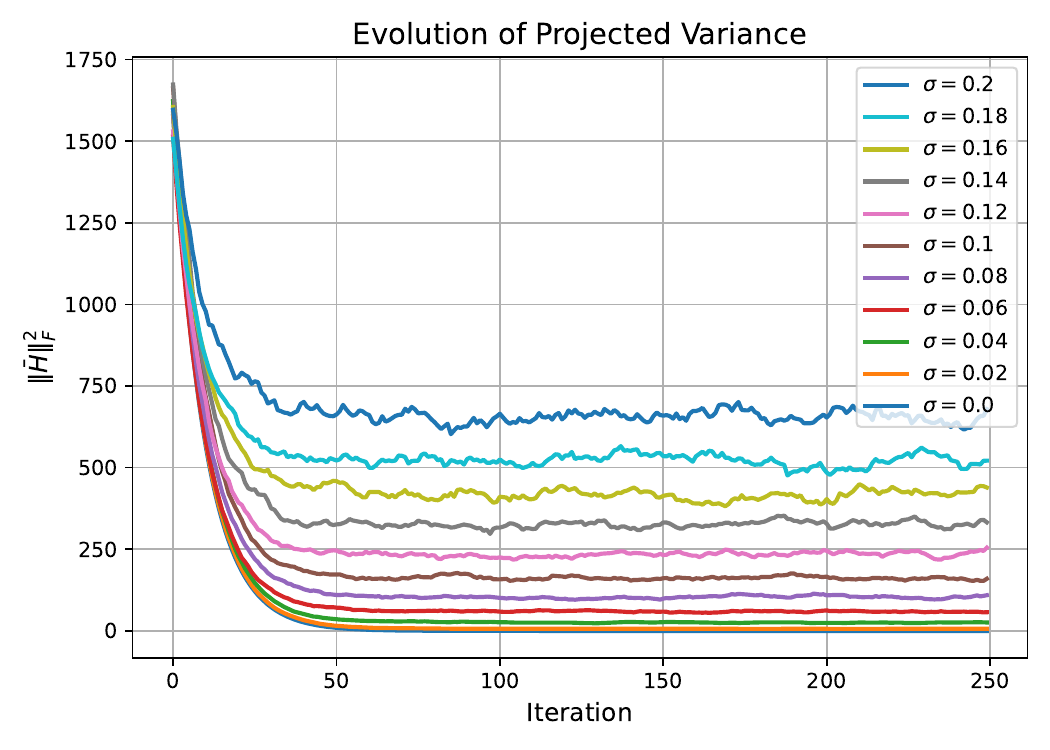}
	\caption{
		Evolution of the projected feature energy
		$\|P H^{(t)}\|_F^2$.
		Persistent Gaussian perturbations prevent collapse of the projected
		variance and produce positive stationary values that increase with the
		noise level, in agreement with
		Theorem~\ref{thm:variance_persistence}.}
	\label{fig:variance_iteration}
\end{figure}

\begin{figure}
	\centering
	\includegraphics[width=0.70\linewidth]{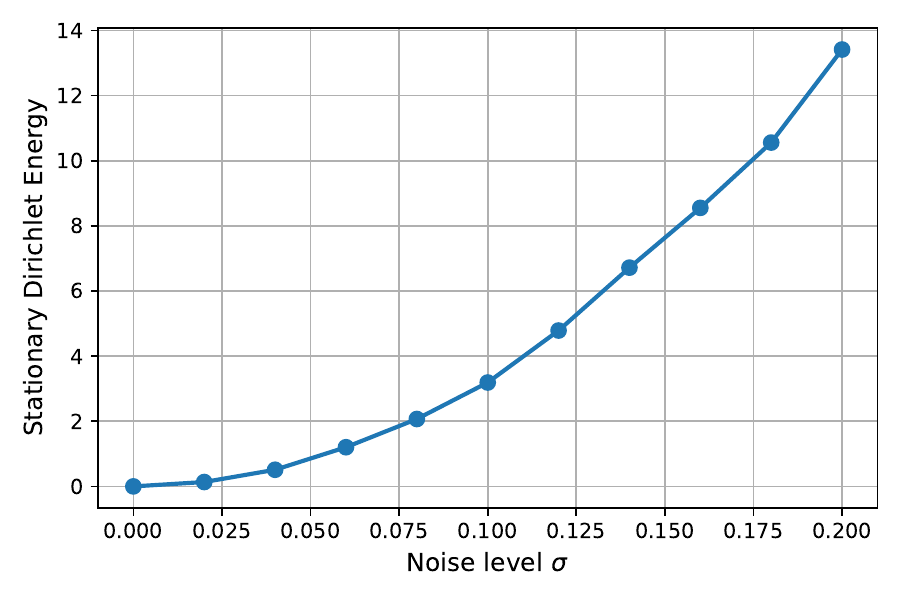}
	\caption{
		Stationary Dirichlet energy as a function of the noise level.
		The equilibrium energy increases monotonically with $\sigma$, validating
		the theoretical prediction that additive noise prevents asymptotic
		oversmoothing.}
	\label{fig:stationary_energy_sigma}
\end{figure}

\begin{figure}
	\centering
	\includegraphics[width=0.70\linewidth]{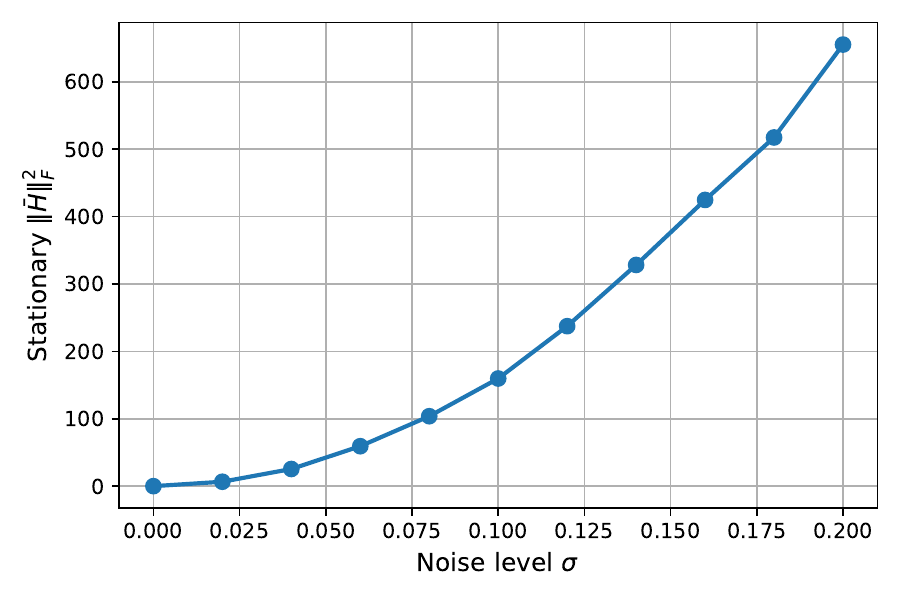}
	\caption{
		Stationary projected feature energy as a function of the noise level.
		The stationary projected variance grows monotonically with the noise
		intensity, confirming the persistence of representation diversity
		predicted by
		Theorem~\ref{thm:variance_persistence}.}
	\label{fig:variance_sigma}
\end{figure}


\subsection{Experiment 2: Nonlinear recurrent graph convolution}

The first experiment validates the theoretical results for the linear
stochastic dynamical system analyzed in Sections
\ref{sec:noisy_dynamics} and
\ref{sec:Asymptotic_Analysis}.
We next investigate whether the same qualitative phenomenon appears in a
practical nonlinear graph neural network.

The theoretical analysis is formulated for a general contractive
stochastic recurrence
\[
H^{(t+1)} = F(H^{(t)}) + \sigma\,\Xi^{(t)},
\]
without assuming a particular graph neural network architecture.
To examine whether the predicted stationary behavior extends beyond the
linear model, we instantiate the recurrence using the standard graph
convolution implemented by the
\texttt{GCNConv} layer of PyTorch Geometric.

The propagation operator employed by this layer is
\[
\hat A
=
\hat D^{-1/2}(A+I)\hat D^{-1/2},
\]
where $\hat D$ denotes the degree matrix of $A+I$.
Unlike the theoretical model, the implementation does not introduce an
explicit contraction coefficient $\alpha$.
Instead, information propagation is governed by the normalized graph
convolution operator together with the learned weight matrices and the
ReLU activation function.
Consequently, this experiment is intended to assess the qualitative
predictions of the theory rather than to exactly reproduce every
assumption used in the proofs.

The recurrent update implemented in the experiment is
\[
H^{(t+1)}
=
\operatorname{ReLU}
\!\left(
\operatorname{GCNConv}(H^{(t)})
\right)
+
\sigma\,\Xi^{(t)},
\]
where
\[
\operatorname{GCNConv}(H)
=
\hat D^{-1/2}(A+I)\hat D^{-1/2}HW,
\]
up to the internal implementation of the PyTorch Geometric layer.
The hidden dimension is fixed to $64$.
Node features are initialized using a learned linear embedding of the
original Cora attributes, after which the same graph convolution is
applied recurrently for $200$ iterations.
The graph, noise levels, and evaluation protocol are identical to those
used in Experiment~1.

As in the linear experiment, the quantities of interest are the
Dirichlet energy
\[
E(H^{(t)})
=
\frac{1}{m}\operatorname{Tr}\bigl((H^{(t)})^{\top} L H^{(t)}\bigr),
\]
computed using the combinatorial graph Laplacian, and the projected
feature energy
\[
\|P H^{(t)}\|_F^2.
\]

Figure~\ref{fig:energy_plot}
shows the evolution of the Dirichlet energy.
The deterministic trajectory ($\sigma=0$) again exhibits the strongest
decay toward a nearly constant representation, whereas every positive
noise level converges to a stationary regime with strictly positive
Dirichlet energy.
Increasing the noise intensity consistently raises the equilibrium
energy, demonstrating that the qualitative mechanism identified by the
theoretical analysis persists in the nonlinear network.

Figure~\ref{fig:energy_sigma}
summarizes the stationary Dirichlet energy as a function of the noise
level.
Although the nonlinear architecture does not satisfy all assumptions of
the mathematical model exactly, the empirical relationship remains
strictly monotone.
The stationary energy increases continuously with $\sigma$, indicating
that persistent stochastic perturbations effectively counteract the loss
of geometric diversity caused by repeated graph propagation.

Taken together, these observations suggest that the mechanism established
by the theoretical analysis is robust to nonlinear message-passing
architectures.
While the proofs apply to an abstract contractive stochastic dynamical
system, the recurrent GCN exhibits the same transition from asymptotic
feature collapse in the deterministic setting to a non-degenerate
stationary regime under persistent Gaussian perturbations.


\begin{figure}
	\centering
	\includegraphics[width=0.78\linewidth]{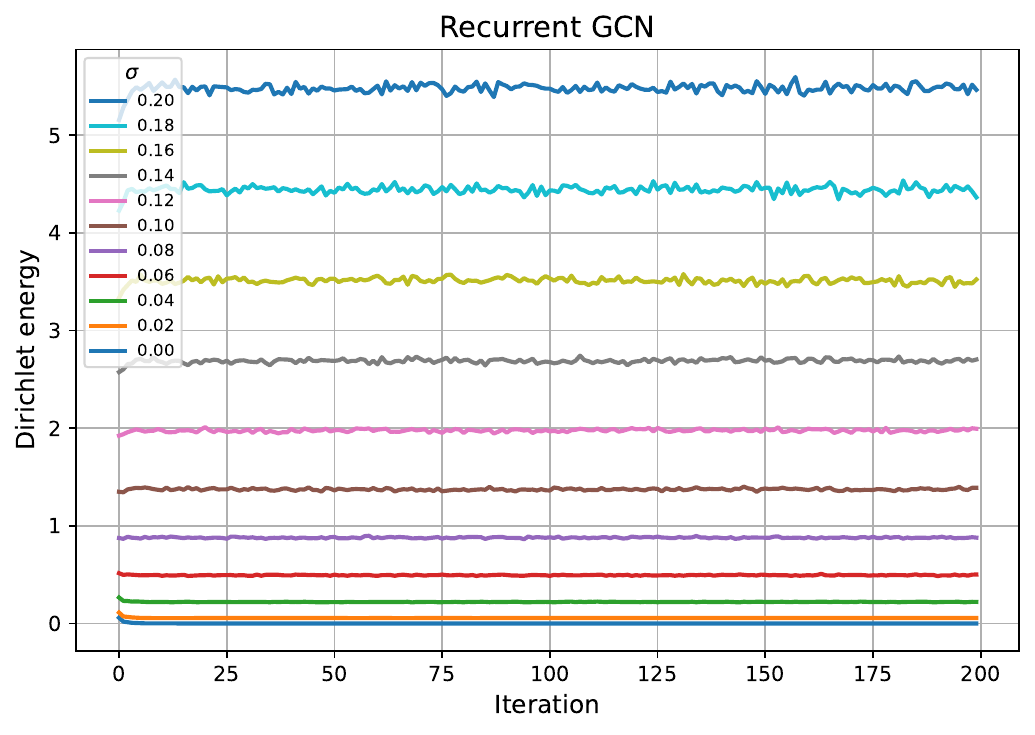}
	\caption{
		Evolution of the Dirichlet energy in the nonlinear recurrent graph
		convolutional network.
		The deterministic trajectory ($\sigma=0$) exhibits the strongest decay,
		whereas persistent Gaussian perturbations maintain strictly positive
		stationary energy.
		Increasing the noise level produces progressively larger equilibrium
		energies, consistent with the qualitative predictions of the theory.
	}
	\label{fig:energy_plot}
\end{figure}

\begin{figure}[t]
	\centering
	\includegraphics[width=0.70\linewidth]{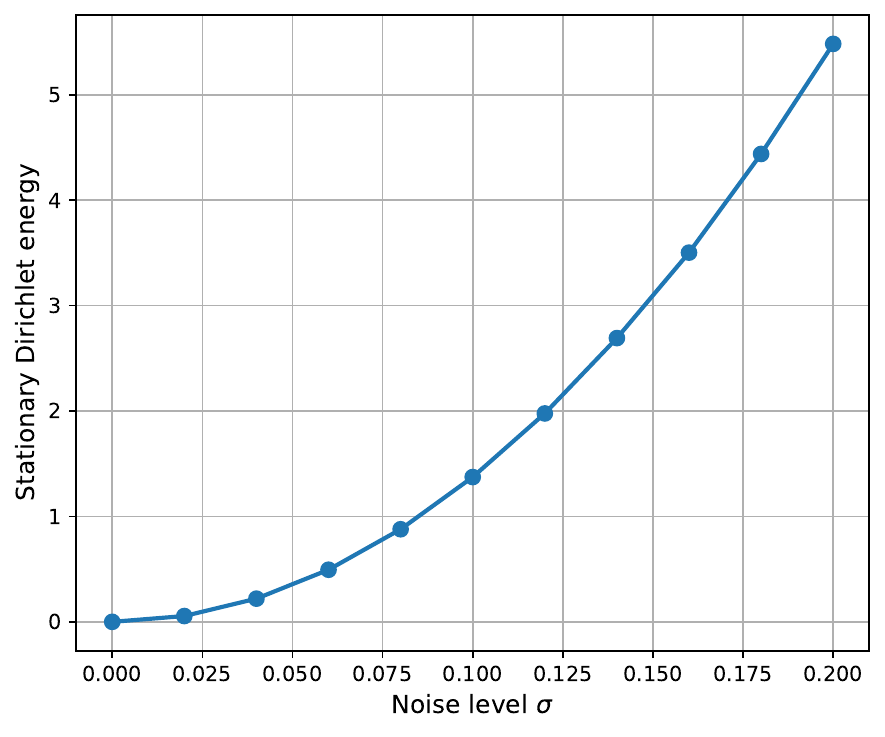}
	\caption{
		Stationary Dirichlet energy of the nonlinear recurrent GCN as a function
		of the noise level.
		Although the theoretical analysis is derived for an abstract contractive
		dynamical system, the practical GCN implementation exhibits the same
		qualitative monotone dependence of the stationary energy on the noise
		intensity.
	}
	\label{fig:energy_sigma}
\end{figure}


\subsection{Experiment 3: Verification of quadratic scaling}

The theoretical analysis predicts that the stationary Dirichlet energy
grows quadratically with the noise intensity.
More precisely, Theorem~\ref{thm:main} establishes that
\[
E_\infty=\mathcal{O}(\sigma^2),
\]
and, for the linear contractive model,
\[
E_\infty
\propto
\frac{\lambda_2}{1-\alpha^2}\sigma^2.
\]

To verify this prediction empirically, we repeated the recurrent GCN
experiment over twenty independent random seeds for each noise level.
For every run, the stationary Dirichlet energy was estimated by averaging
the final $150$ iterations after discarding an initial burn-in period of
$100$ iterations.
The resulting stationary energies were then regressed against
$\sigma^2$ using ordinary least squares.
The fitted regression is
\[
E_\infty
\approx
137.86\,\sigma^2
+
1.29\times10^{-4},
\]
with coefficient of determination
\(
R^2=1.0000
\).
The slope $137.86$ reflects the combined influence of the graph size,
spectral gap, feature dimension, and contraction factor through the
theoretical relation $\frac{m\lambda_2(N-1)d}{1-\alpha^2}$, although the
nonlinear GCN does not enforce a fixed contraction parameter
$\alpha$, so the numerical value is not directly predicted.
The essentially perfect linear fit and negligible intercept confirm the
predicted quadratic dependence of the stationary Dirichlet energy on
$\sigma^2$.

Figure~\ref{fig:quadratic_scaling}
shows both the empirical stationary energies and the fitted regression
line.
The data points lie almost perfectly on a straight line,
demonstrating an excellent agreement with the theoretical prediction.
The fitted intercept is numerically negligible, indicating that the
stationary Dirichlet energy continuously approaches zero as
$\sigma\rightarrow0$, while increasing proportionally to $\sigma^2$ for
positive noise levels.
These results provide strong quantitative evidence that the stationary
Dirichlet energy follows the quadratic scaling law predicted by the
theoretical analysis.

\begin{figure}
	\centering
	\includegraphics[width=0.75\linewidth]{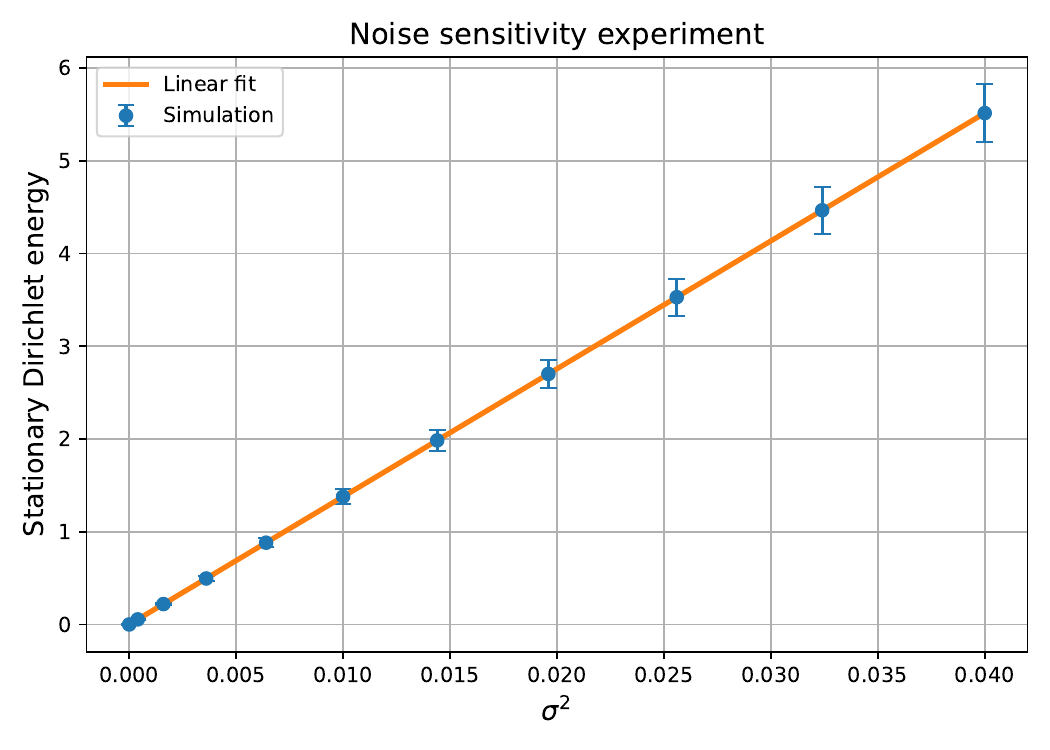}
	\caption{Verification of the quadratic scaling law.
		The stationary Dirichlet energy obtained from repeated recurrent GCN
		simulations is plotted against $\sigma^2$ together with the least-squares
		linear fit.
		The almost perfect agreement ($R^2=1.0000$) confirms the theoretical
		prediction that the stationary Dirichlet energy grows proportionally to
		$\sigma^2$.}
	\label{fig:quadratic_scaling}
\end{figure}

\subsection{Discussion of the experimental results}

The numerical experiments consistently support the theoretical analysis
developed in the previous sections and provide empirical evidence for both
the qualitative and quantitative predictions of the proposed stochastic
framework.

Experiment~1 validates the behavior of the abstract contractive stochastic
dynamical system studied throughout the theoretical analysis.
The deterministic dynamics ($\sigma=0$) exhibit the expected progressive
loss of feature diversity, with both the Dirichlet energy and the
projected feature energy converging toward zero.
Introducing persistent Gaussian perturbations fundamentally changes the
long-term behavior: every positive noise level produces a non-degenerate
stationary regime with strictly positive Dirichlet energy and projected
variance.
Moreover, both stationary quantities increase monotonically with the noise
intensity, in agreement with
Theorem~\ref{thm:variance_persistence} and
Theorem~\ref{thm:main}.

Experiment~2 demonstrates that the same qualitative phenomenon also
appears in a nonlinear recurrent graph convolutional network implemented
using the standard \texttt{GCNConv} layer.
Although the theoretical analysis is formulated for an abstract
contractive update map and the implementation does not explicitly include
the contraction parameter $\alpha$, the recurrent GCN exhibits the same
transition from deterministic feature collapse to a stationary regime with
persistent geometric diversity under additive noise.
Furthermore, the stationary Dirichlet energy again increases
monotonically with the noise level, suggesting that the mechanism
identified by the theoretical analysis is not limited to the simplified
linear model but also manifests itself in practical message-passing graph
neural networks.

Experiment~3 provides quantitative verification of the asymptotic scaling
predicted by the theory.
Theorem~\ref{thm:variance_persistence} establishes that the stationary
projected variance scales proportionally to
\[
\frac{\sigma^2}{1-\alpha^2},
\]
which implies the quadratic dependence of the stationary Dirichlet energy
derived in Theorem~\ref{thm:main}.
The regression performed in Experiment~3 yields
\[
E_\infty
\approx
137.86\,\sigma^2
+
1.29\times10^{-4},
\]
with an essentially perfect coefficient of determination
\(
R^2=1.0000\).
The negligible intercept together with the excellent linear fit confirms
that the stationary Dirichlet energy is accurately described by a
quadratic function of the noise intensity over the investigated range,
providing strong numerical support for the theoretical scaling law.

Taken together, the three experiments validate complementary aspects of
the theoretical framework.
Experiment~1 confirms the existence of a non-degenerate stationary regime
for the abstract stochastic propagation model.
Experiment~2 demonstrates that the same qualitative behavior persists in a
nonlinear recurrent graph neural network implemented using standard graph
convolution.
Experiment~3 verifies the predicted quadratic dependence of the stationary
Dirichlet energy on the noise variance.
Although the experiments are intended as theoretical validation rather
than benchmark comparisons against competing GNN architectures, they
collectively provide strong empirical evidence that persistent additive
Gaussian perturbations prevent asymptotic oversmoothing while preserving
nontrivial geometric information in the learned node representations.

\section{Discussion}
\label{sec:Discussion}

The present work provides a stochastic perspective on the oversmoothing
phenomenon in recurrent graph neural networks.
Rather than analyzing oversmoothing solely through deterministic message
passing, we study the long-term behavior of noisy recurrent dynamics and
show that persistent stochastic perturbations fundamentally alter the
asymptotic regime of the hidden representations.

Our main theoretical contribution is the identification of a stationary
mechanism that prevents complete feature collapse.
Under a global contraction assumption, the projected hidden
representations evolve as an ergodic Markov chain possessing a unique
invariant probability measure.
The combination of ergodicity, persistent stochastic excitation, and the
spectral properties of the graph Laplacian yields a strictly positive
lower bound on the stationary Dirichlet energy.
Consequently, the asymptotic hidden representations remain separated from
the consensus manifold, establishing escape from oversmoothing in a
rigorous probabilistic sense.

The proposed framework also provides a quantitative interpretation of the
role of additive noise.
The lower bound
\[
\mathbb E_\mu[\mathcal E(H)]
\ge
\frac{\lambda_2(L)}{m}
\frac{\sigma^2(N-1)d}{1-\alpha^2}
\]
reveals how the limiting representation quality depends jointly on the
graph topology, the contraction strength of the deterministic dynamics,
and the noise intensity.
In particular, stronger graph connectivity (larger algebraic
connectivity), weaker contraction, or larger noise levels all increase
the minimum achievable stationary Dirichlet energy.

The numerical experiments support both the qualitative and quantitative
predictions of the theory.
The simulated recurrent dynamics exhibit the transition from complete
oversmoothing to a stationary regime with persistent representation
variation, while the empirical stationary Dirichlet energy grows
approximately quadratically with the noise intensity, in agreement with
the theoretical scaling law.

The analysis presented here is intentionally general.
The deterministic update map is not restricted to a specific graph neural
network architecture, allowing the results to apply to a broad class of
contractive recurrent message-passing models.
At the same time, this level of abstraction leaves several important
questions open:
\begin{itemize}
\item 
First, the global contraction assumption is mathematically convenient but
stronger than what is typically satisfied by practical graph neural
networks.
Many modern architectures exhibit only local contraction, average
contraction, or contraction restricted to particular regions of the state
space.
Extending the present analysis to weaker notions of stability therefore
constitutes an important direction for future research.
\item 
Second, the present work considers additive Gaussian perturbations with
constant variance.
Although Gaussian noise leads to an explicit stationary variance identity,
other forms of stochastic perturbation, including state-dependent,
anisotropic, or multiplicative noise, may induce qualitatively different
stationary behavior and deserve further investigation.
\item 
Finally, our analysis focuses exclusively on the geometric structure of
the hidden representations.
While preventing oversmoothing is generally associated with improved
representation diversity, the relationship between stationary Dirichlet
energy and downstream predictive performance remains an empirical
question.
Establishing rigorous links between the probabilistic theory developed
here and supervised learning objectives represents an interesting
direction for future work.
\end{itemize}

Overall, the results suggest that stochastic perturbations should not be
viewed merely as regularization techniques but rather as dynamical
mechanisms capable of fundamentally changing the asymptotic behavior of
recurrent graph neural networks.
This probabilistic viewpoint opens new possibilities for studying graph
representation learning through the theory of stochastic dynamical systems
and invariant measures.

\section{Conclusion}
\label{sec:Conclusion}

This paper developed a stochastic framework for analyzing oversmoothing in
recurrent graph neural networks with additive noise.
By projecting the dynamics onto the zero-mean representation space, the
problem was formulated as a contractive Markov chain whose long-term
behavior can be studied using tools from stochastic dynamical systems.
Under mild structural assumptions, we proved that the projected dynamics
admit a unique invariant probability measure and established an explicit
lower bound on the stationary variance induced by additive Gaussian
perturbations.
Combining this probabilistic result with a spectral inequality for the
graph Laplacian yielded a quantitative lower bound on the asymptotic
Dirichlet energy, proving that the noisy recurrent dynamics escape
oversmoothing.
The analysis also revealed the explicit dependence of the limiting energy
on the graph topology, the contraction rate, and the noise intensity.
Numerical experiments corroborated the theoretical predictions.
In particular, they demonstrated that additive noise prevents complete
feature collapse, that the hidden representations converge to a
non-degenerate stationary regime, and that the stationary Dirichlet
energy exhibits the predicted quadratic dependence on the noise level.

Beyond the specific setting considered here, the proposed framework
illustrates how probabilistic methods can provide rigorous insight into
the asymptotic behavior of graph neural networks.
We hope that this perspective will motivate further research on
stochastic graph learning models, weaker stability assumptions, and more
general classes of noisy message-passing architectures.

\bibliographystyle{plain}
\bibliography{references}

@inproceedings{kipf2017,
  author    = {Thomas N. Kipf and Max Welling},
  title     = {Semi-supervised classification with graph convolutional networks},
  booktitle = {International Conference on Learning Representations},
  year      = {2017},
}

@inproceedings{hamilton2017,
  author    = {William Hamilton and Zhitao Ying and Jure Leskovec},
  title     = {Inductive representation learning on large graphs},
  booktitle = {Advances in Neural Information Processing Systems},
  year      = {2017},
}

@inproceedings{velivckovic2018,
  author    = {Petar Veli\v{c}kovi\'{c} and Guillem Cucurull and Arantxa Casanova and
               Adriana Romero and Pietro Li\`{o} and Yoshua Bengio},
  title     = {{Graph Attention Networks}},
  booktitle = {International Conference on Learning Representations},
  year      = {2018},
}

@inproceedings{xu2019gin,
  author    = {Keyulu Xu and Weihua Hu and Jure Leskovec and Stefanie Jegelka},
  title     = {How powerful are graph neural networks?},
  booktitle = {International Conference on Learning Representations},
  year      = {2019},
}

@inproceedings{li2018deeper,
  author    = {Qimai Li and Zhichao Han and Xiao-Ming Wu},
  title     = {Deeper insights into graph convolutional networks for semi-supervised learning},
  booktitle = {Proceedings of the AAAI Conference on Artificial Intelligence},
  year      = {2018},
}

@inproceedings{oono2020graph,
  author    = {Kenta Oono and Taiji Suzuki},
  title     = {Graph neural networks exponentially lose expressive power for node classification},
  booktitle = {International Conference on Learning Representations},
  year      = {2020},
}

@misc{nt2020revisiting,
  author        = {Hoang NT and Tyler Maehara},
  title         = {Revisiting Over-smoothing in Graph Neural Networks},
  year          = {2020},
  eprint        = {2003.13663},
  archivePrefix = {arXiv},
  primaryClass  = {cs.LG},
  url           = {https://arxiv.org/abs/2003.13663}
}

@inproceedings{zhao2020pairnorm,
  author    = {Lingxiao Zhao and Leman Akoglu},
  title = {{PairNorm}: Tackling Oversmoothing in {GNN}s},
  booktitle = {International Conference on Learning Representations},
  year      = {2020},
}

@inproceedings{rong2020dropedge,
  author    = {Yu Rong and Wenbing Huang and Tingyang Xu and Junzhou Huang},
  title = {{DropEdge}: Towards Deep Graph Convolutional Networks on Node Classification},
  booktitle = {International Conference on Learning Representations},
  year      = {2020},
}

@inproceedings{chen2020gcnii,
  author    = {Ming Chen and Zhe Wei and Zengfeng Huang and Bolin Ding and Yaliang Li},
  title = {Simple and Deep Graph Convolutional Networks},
  booktitle = {International Conference on Machine Learning},
  year      = {2020}
}

@inproceedings{klicpera2019predict,
  author    = {Johannes Klicpera and Stefan Wei{\ss}enberger and Stephan G{\"u}nnemann},
  title     = {Predict then propagate: Graph neural networks meet personalized PageRank},
  booktitle = {International Conference on Learning Representations},
  year      = {2019}
}

@inproceedings{xu2018jknet,
  author    = {Keyulu Xu and Chengtao Li and Yonglong Tian and Tomohiro Sonobe and
               Ken-ichi Kawarabayashi and Stefanie Jegelka},
  title     = {Representation learning on graphs with Jumping Knowledge Networks},
  booktitle = {International Conference on Machine Learning},
  year      = {2018}
}

@misc{li2020deepergcn,
  author        = {Guohao Li and Chenxin Xiong and Ali Thabet and Bernard Ghanem},
  title         = {{DeeperGCN}: All You Need to Train Deeper {GCN}s},
  year          = {2020},
  eprint        = {2006.07739},
  archivePrefix = {arXiv},
  primaryClass  = {cs.LG},
  url           = {https://arxiv.org/abs/2006.07739}
}

@inproceedings{cai2021graphnorm,
  author    = {Tianle Cai and Shengjie Luo and Keyulu Xu and Di He and Tie-Yan Liu and Liwei Wang},
  title = {{GraphNorm}: A Principled Approach to Accelerating Graph Neural Network Training},
  booktitle = {International Conference on Machine Learning},
  year      = {2021},
}

@inproceedings{elinas2023revisiting,
  author    = {Panos Elinas and Edwin Bonilla},
  title     = {Revisiting over-smoothing in graph neural networks},
  booktitle = {International Conference on Learning Representations},
  year      = {2023},
}

@inproceedings{scholkemper2025,
  author    = {Michael Scholkemper and Xinyi Wu and Ali Jadbabaie and Soledad Villar},
  title = {Residual Connections and Normalization Can Provably Prevent Oversmoothing in Graph Neural Networks},
  booktitle = {International Conference on Learning Representations},
  year      = {2025},
}

@inproceedings{alon2021bottleneck,
  author    = {Uri Alon and Eran Yahav},
  title     = {On the bottleneck of graph neural networks and its practical implications},
  booktitle = {International Conference on Learning Representations},
  year      = {2021},
}

@inproceedings{topping2022understanding,
  author    = {Jake Topping and Francesco Di Giovanni and Benjamin Paul Chamberlain and Xiaowen Dong and Michael M. Bronstein},
  title     = {Understanding over-squashing and bottlenecks on graphs via curvature},
  booktitle = {International Conference on Learning Representations},
  year      = {2022},
}

@article{digiovanni2023oversquashing,
  author    = {Francesco Di Giovanni and Lorenzo Giusti and Federico Barbero and Giulia Luise and Pietro Li\`{o} and Michael M. Bronstein},
  title     = {How does over-squashing affect graph neural networks?},
  journal   = {Transactions on Machine Learning Research},
  year      = {2023},
}

@inproceedings{chamberlain2021grand,
  author    = {Benjamin Paul Chamberlain and James Rowbottom and Maria I. Gorinova and Stefan Webb and Emanuele Rossi and Michael M. Bronstein},
  title     = {{GRAND}: Graph Neural Diffusion},
  booktitle = {International Conference on Machine Learning},
  year      = {2021},
}

@inproceedings{sato2021random,
  author    = {Ryoma Sato and Makoto Yamada and Hisashi Kashima},
  title     = {Random features strengthen graph neural networks},
  booktitle = {International Conference on Machine Learning},
  year      = {2021},
}

@inproceedings{godwin2022simple,
  author={Godwin, Jonathan and Schaarschmidt, Michael and Gaunt, Alexander and Sanchez-Gonzalez, Alvaro and Rubanova, Yulia and Veli{\v{c}}kovi{\'c}, Petar and Kirkpatrick, James and Battaglia, Peter},
  title     = {Simple {GNN} Regularisation for {3D} Molecular Property Prediction and Beyond},
  booktitle = {International Conference on Learning Representations},
  year      = {2022},
  url       = {https://openreview.net/forum?id=9X-hgLDLYkQ}
}

@misc{zhuang2020sdgnn,
  author        = {Juntang Zhuang and Nicha C. Dvornek and Xiaoxiao Li and
                   Junzhou Yang and James S. Duncan},
  title         = {{SDGNN}: Learning Stochastic Graph Neural Networks},
  year          = {2020},
  eprint        = {2002.07836},
  archivePrefix = {arXiv},
  primaryClass  = {cs.LG},
  url           = {https://arxiv.org/abs/2002.07836}
}

@inproceedings{chen2018neuralode,
  author    = {Ricky T. Q. Chen and Yulia Rubanova and Jesse Bettencourt and David Duvenaud},
  title     = {Neural Ordinary Differential Equations},
  booktitle = {Advances in Neural Information Processing Systems},
  year      = {2018},
}

@misc{poli2019graphode,
  author       = {Michael Poli and
                  Stefano Massaroli and
                  Junyoung Park and
                  Atsushi Yamashita and
                  Hajime Asama and
                  Jinkyoo Park},
  title        = {Graph Neural Ordinary Differential Equations},
  year         = {2019},
  eprint       = {1911.07532},
  archivePrefix= {arXiv},
  primaryClass = {cs.LG},
  url          = {https://arxiv.org/abs/1911.07532}
}

@inproceedings{bai2019deep,
  author    = {Shaojie Bai and J. Zico Kolter and Vladlen Koltun},
  title     = {Deep equilibrium models},
  booktitle = {Advances in Neural Information Processing Systems},
  year      = {2019},
}

@inproceedings{gu2020implicit,
 author = {Gu, Fangda and Chang, Heng and Zhu, Wenwu and Sojoudi, Somayeh and El Ghaoui, Laurent},
 booktitle = {Advances in Neural Information Processing Systems},
 pages = {11984--11995},
 publisher = {Curran Associates, Inc.},
 title = {Implicit Graph Neural Networks},
 volume = {33},
 year = {2020}
}

@inproceedings{rusch2023graphcon,
  author    = {T. Konstantin Rusch and Michael M. Bronstein and Siddhartha Mishra},
  title     = {Graph-Coupled Oscillator Networks},
  booktitle = {International Conference on Machine Learning},
  year      = {2023},
}

@inproceedings{ying2021graphormer,
  author    = {Chengxuan Ying and Tianle Cai and Shengjie Luo and Shuxin Zheng and Guolin Ke and Di He and Yanming Shen and Tie-Yan Liu},
  title = {Do Transformers Really Perform Bad for Graph Representation?},
  booktitle = {Advances in Neural Information Processing Systems},
  year      = {2021},
}

@inproceedings{rampasek2022graphgps,
  author    = {Ladislav Ramp{\'a}{\v{s}}ek and Michael Galkin and Vijay Prakash Dwivedi and Anh Tuan Luu and Guy Wolf and Dominique Beaini},
  title     = {Recipe for a General, Powerful, Scalable Graph Transformer},
  booktitle = {Advances in Neural Information Processing Systems},
  year      = {2022},
}

@book{chung1997,
  author    = {Fan R. K. Chung},
  title     = {Spectral Graph Theory},
  series    = {CBMS Regional Conference Series in Mathematics},
  volume    = {92},
  publisher = {American Mathematical Society},
  year      = {1997},
  isbn      = {978-0-8218-0315-8}
}

@article{shuman2013,
  author  = {David I. Shuman and Sunil Narang and Pascal Frossard and
             Antonio Ortega and Pierre Vandergheynst},
  title   = {The Emerging Field of Signal Processing on Graphs},
  journal = {IEEE Signal Processing Magazine},
  volume  = {30},
  number  = {3},
  pages   = {83--98},
  year    = {2013},
  doi     = {10.1109/MSP.2012.2235192}
}

@inproceedings{balcilar2021analyzing,
  author    = {Muhammet Balcilar and Guillaume Renton and Pierre H\'{e}roux and Beno\^{i}t Ga\"{u}z\`{e}re and S\'{e}bastien Adam and Paul Honeine},
  title     = {Analyzing the expressive power of graph neural networks in a spectral perspective},
  booktitle = {International Conference on Learning Representations},
  year      = {2021},
}

@article{diaconis1999iterated,
  author  = {Persi Diaconis and David Freedman},
  title   = {Iterated Random Functions},
  journal = {SIAM Review},
  volume  = {41},
  number  = {1},
  pages   = {45--76},
  year    = {1999},
  doi     = {10.1137/S0036144598338446}
}

@book{arnold1998,
  author    = {Ludwig Arnold},
  title     = {Random Dynamical Systems},
  series    = {Springer Monographs in Mathematics},
  publisher = {Springer},
  address   = {Berlin},
  year      = {1998},
  isbn      = {978-3-540-63758-5}
}

@book{meyn2009,
  author    = {Sean P. Meyn and Richard L. Tweedie},
  title     = {Markov Chains and Stochastic Stability},
  edition   = {2},
  publisher = {Cambridge University Press},
  year      = {2009},
  isbn      = {978-0-521-73182-9}
}

@incollection{hairer2011,
  author    = {Martin Hairer and Jonathan C. Mattingly},
  title     = {Yet Another Look at Harris' Ergodic Theorem for Markov Chains},
  booktitle = {Seminar on Stochastic Analysis, Random Fields and Applications VI},
  editor    = {Robert C. Dalang and Marco Dozzi and Francesco Russo},
  series    = {Progress in Probability},
  volume    = {63},
  pages     = {109--117},
  publisher = {Birkh{\"a}user},
  year      = {2011},
  doi       = {10.1007/978-3-0348-0027-9_6}
}

@book{villani2003,
  author    = {C{\'e}dric Villani},
  title     = {Topics in Optimal Transportation},
  series    = {Graduate Studies in Mathematics},
  volume    = {58},
  publisher = {American Mathematical Society},
  address   = {Providence, RI},
  year      = {2003},
  isbn      = {978-0821833124}
}

@book{villani2009,
  author    = {C{\'e}dric Villani},
  title     = {Optimal Transport: Old and New},
  series    = {Grundlehren der mathematischen Wissenschaften},
  volume    = {338},
  publisher = {Springer},
  address   = {Berlin},
  year      = {2009},
  isbn      = {978-3-540-71049-3},
  doi       = {10.1007/978-3-540-71050-9}
}

@article{hoseinnia2025,
  author  = {Fateme Hoseinnia and Mehdi Ghatee and Mostafa Haghir Chehreghani},
  title   = {Mitigating over-smoothing in Graph Neural Networks for Node Classification through Adaptive Early Embedding and Biased {DropEdge} Procedures},
  journal = {Knowledge-Based Systems},
  volume  = {320},
  pages   = {113615},
  year    = {2025},
  doi     = {10.1016/j.knosys.2025.113615}
}

@article{nasrabadi2025,
  author  = {Fatemeh Gholamzadeh Nasrabadi and AmirHossein Kashani and
             Pegah Zahedi and Mostafa Haghir Chehreghani},
  title   = {Content Augmented Graph Neural Networks},
  journal = {ACM Transactions on the Web},
  volume  = {19},
  number  = {4},
  pages   = {40:1--40:19},
  year    = {2025},
  doi     = {10.1145/3733510}
}

@article{zohrabi2024,
  author  = {M. Zohrabi and S. Saravani and Mostafa Haghir Chehreghani},
  title   = {Centrality-based and Similarity-based Neighborhood Extension in Graph Neural Networks},
  journal = {The Journal of Supercomputing},
  volume  = {80},
  number  = {16},
  pages   = {24638--24663},
  year    = {2024},
  doi     = {10.1007/s11227-024-06285-z}
}

@article{mohamadi2025oversquashing,
  author  = {Yassin Mohamadi and Mostafa Haghir Chehreghani},
  title   = {Mitigating Over-squashing in Graph Few-shot Learning by Leveraging Local and Global Similarities},
  journal = {Applied Soft Computing},
  volume  = {184},
  pages   = {113863},
  year    = {2025},
  doi     = {10.1016/j.asoc.2025.113863}
}

@article{HaghirChehreghani2022Nature,
  author    = {Mostafa Haghir Chehreghani},
  title     = {Half a decade of graph convolutional networks},
  journal   = {Nature Machine Intelligence},
  volume    = {4},
  number    = {3},
  pages     = {192--193},
  year      = {2022}
}

\end{document}